\documentclass[a4paper,10pt]{article}

\usepackage[utf8]{inputenc}
\usepackage{hyperref}
\usepackage{amsmath}
\usepackage{amssymb}
\usepackage{amsthm}
\usepackage{thmtools,thm-restate}
\usepackage{bm}
\usepackage{algorithm} 
\usepackage{algpseudocode}
\usepackage{mathtools}
\usepackage{aligned-overset}
\usepackage{comment}
\usepackage{cleveref}

\usepackage{xcolor}
\usepackage{enumerate}
\usepackage{breqn}
\usepackage{subcaption}

\usepackage{doi}
\usepackage{authblk}
\usepackage{tikz}
\usepackage{charter}

\usepackage{macros}

\definecolor{labelkey}{rgb}{0,0.08,0.45}
\definecolor{refkey}{rgb}{0,0.6,0.0}
\definecolor{Brown}{rgb}{0.45,0.0,0.05}
\definecolor{dgreen}{rgb}{0.00,0.49,0.00}
\definecolor{dblue}{rgb}{0,0.08,0.75}

\hypersetup{
colorlinks,
linkcolor=dgreen,
citecolor=dblue
}

\usepackage[round]{natbib}

\title{ { \sffamily General Tail Bounds for Non-Smooth Stochastic Mirror Descent
} }

\author[]{Khaled Eldowa}
\author[]{Andrea Paudice}

\affil[]{\footnotesize Department of Computer Science, University of Milan, Via Giovanni Celoria 18, 20133 Milano, Italy}

\date{}

\begin{document}
\maketitle

\begin{abstract}
In this paper, we provide novel tail bounds on the optimization error of Stochastic Mirror Descent for convex and Lipschitz objectives. Our analysis extends the existing tail bounds from the classical light-tailed Sub-Gaussian noise case to heavier-tailed noise regimes. We study the optimization error of the last iterate as well as the average of the iterates. We instantiate our results in two important cases: a class of noise with exponential tails and one with polynomial tails. A remarkable feature of our results is that they do not require an upper bound on the diameter of the domain. Finally, we support our theory with illustrative experiments that compare the behavior of the average of the iterates with that of the last iterate in heavy-tailed noise regimes.
\end{abstract}

 \vspace{1ex}
\noindent
%{\bf\small Keywords.} 
%{\small 
%Stochastic convex optimization, 
%high-probability bounds, 
%subgradient method, heavy tailed noise.}\\[1ex]
%\noindent
%{\bf\small AMS Mathematics Subject Classification:} {\small 90C25, 65K05, 49M37, 90C15, 90C06}

\section{Introduction}
\emph{Stochastic Mirror Descent} (SMD) and its more popular Euclidean counterpart \emph{Stochastic (sub-)Gradient Descent} (SGD) are at the core of modern machine learning. For example, they are widely used for performing large-scale optimization tasks, as in the case of empirical (or regularized) risk minimization, and for minimizing the statistical risk in kernel methods. In this paper, we study the performance of SMD in the general problem of minimizing a (non-smooth) convex and Lipschitz function given only noisy oracle access to its (sub-)gradients.
SGD was first introduced by \cite{Ermoliev1969},
who studied the convergence of the iterates for convex Lipschitz objectives. Subsequent studies focused on deriving in-expectation bounds on the optimization error of the average of the iterates. Denoting with $T$ the number of iterations, these bounds are of the order of $1/\sqrt{T}$. In their seminal work, \citep{Nemirovski2009} introduced SMD as a non Euclidean generalization of SGD and showed that it enjoys the same $1/\sqrt{T}$ bound. The shortcoming of in-expectation bounds is that they do not offer guarantees on individual runs of the algorithm. This is especially limiting when multiple runs of the algorithm are not possible, as in large scale problems,
or when the data arrives in a stream. Tail bounds offer stronger guarantees that apply to individual runs of the algorithms. For a fixed confidence level $\delta \in(0,1)$, a straightforward application of Markov's inequality, gives a bound of the order $1/(\delta \sqrt{T})$ that holds with probability at least $1-\delta$. This bound is much worse than its in-expectation counterpart, even for moderately small $\delta$. Tighter tail bounds with only an overhead of order $\sqrt{\log(1/\delta)}$ have been obtained under a sub-Gaussian assumption on the noise \citep{Liu2023}. 

Recent works \citep{Zhang2020} show that in some settings, the sub-Gaussian assumption is not appropriate, and the noise is better modelled by heavier tailed distributions. Most works studying tail bounds for SMD (SGD) under heavy-tailed noise consider the extreme cases where the noise is only assumed to have finite variance or  lower order moments (e.g., \cite{Gorbunov2020,Nguyen2023}). Under these assumptions, it is necessary to employ some form of truncation of the (sub-)gradients to obtain a poly-logarithmic dependence on $1/\delta$. Instead, we consider less-studied intermediate regimes for the noise, including two classes of sub-Weibull and polynomially tailed distributions. The former is a class of random variables with exponentially decaying tails (including sub-Gaussian and sub-exponential distributions), which has been shown to be relevant in practical applications \citep{Vladimirova2020}, and has been studied in various machine learning and optimization problems \citep{Madden2020,Kim2022,Li2022,Li2023,Wood2023}. The latter class, which includes some Pareto and power law distributions, has also recently captured interest in the machine learning community \citep{Bakhshizadeh2023,Lou2022}. Moreover, we study the performance of SMD in its plain form (i.e., without truncation), which, in practice, is the more widely used approach. Also, truncation introduces at least one additional parameter, the truncation level, further complicating the tuning of the algorithm in practice. Ideally then, one would like to avoid truncation unless the noise is extremely heavy-tailed.

On a different thread, we notice that most results in the non-smooth case concern the average of the iterates generated by SMD during its execution. However, in practice, taking as solution just the last iterate is by far the preferred heuristic. Consequently, more recent works \citep{Shamir13,Harvey2019,Jain2021} have focused on developing an understanding of the theoretical performance of this approach. In this case, state-of-the-art tails bounds are of almost (up to $\log(T)$ factors)  the same order as for the average of the iterates, though the analyses are still restricted to the sub-Gaussian noise regime.

Motivated by these facts, we derive novel and general tail bounds for both the average of the iterates (\Cref{sec:avg}) and the last iterate (\Cref{sec:last}). In their most general form, our results require controlling the tails of certain martingales depending only on the noise. We then show how to instantiate these bounds in the two considered noise models. Unlike most tail bounds in the (non-smooth) convex and Lipschitz setting, our results do not require a bound on the diameter of the domain. On the technical side, we extend existing analysis techniques and concentration results to cope with the challenges posed by our more general problem setting. In particular, the combination of the heavy-tailed noise with the unbounded domain and the peculiar recurrences arising in the analysis of the last iterate. Finally, some of our results for the average of the iterates show an intriguing \emph{two-regime} phenomenon (also observed in \citep{Lou2022} in a different and more specific setting), where the terms accounting for the heavy-tailed behavior of the noise decay more quickly with the horizon $T$. As our results for the last iterate do not exhibit this behavior, we investigate further this separation in the experiments (\Cref{sec:exp}).

\section{Related Works}
In the case of \emph{sub-Gaussian noise}, the performance of SGD and SMD has been analyzed in \citep{Harvey2019,Jain2021} and \citep{Liu2023} respectively. In \citep{Harvey2019}, the authors consider the setting with a \emph{bounded domain}. They provide tail bounds for  the average of the iterates and the last iterate of the order $\sqrt{\log(1/\delta)/T}$ and  $\sqrt{\log(1/\delta)/T} \cdot \log(T)$ respectively. \cite{Jain2021} show that when the time horizon is known in advance, the last iterate enjoys the same tail bound as the average of the iterates as long as a carefully designed step-size schedule is used. Notably, this result does hold for \emph{unbounded domains}. \cite{Liu2023} consider the more general framework of SMD with \emph{unbounded domains}, and analyze the performance of the  average of the iterates. The authors prove tail bounds of the order of $\sqrt{\log(1/\delta)/T}$ and $\sqrt{\log(1/\delta)/T} \cdot \log(T)$ for the case of known and unknown $T$ respectively.

On the other extreme of the spectrum, another research line considers very general models where the noise is only assumed to posses moments of order at most $p \in (1, 2]$. All these works, consider modifications of the standard SMD where the oracle answers are pre-processed via some form of \emph{truncation}. Truncation schemes allow one to obtain bounds of order $\log(1/\delta)/T^{(p-1)/p}$, regardless of the specific distribution of the noise.
For $p=2$, \cite{Parletta2022} provide tail bounds for several averaging schemes under the assumption of a \emph{bounded domain}, where both the cases of known and unknown $T$ are considered. The \emph{unbounded domain} setting is analyzed in \citep{Gorbunov2021}, although only in the case when $T$ in known. Similar results have been obtained for \emph{smooth} convex objectives \citep{Nazin2019,Gorbunov2020,Holland2022,Nguyen2023}, where both \emph{bounded} and \emph{unbounded domains} have been considered. 
Our work is conceptually close to that of \cite{Lou2022}, which explores the limits of plain SGD in the specific problem of least-squares regression with linear models. In that paper, the authors derive tails bounds for the average of the iterates under polynomially-tailed noise. We recover similar results in our more general problem setting, including the \emph{two-regime} behavior highlighted therein.

\section{Problem Setting}
We consider the problem of minimizing a convex function $f \colon \Xs \rightarrow \R$, where the domain $\Xs \subseteq \R^d$ is a non-empty, closed, and convex set over which $f$ admits a minimum. 
For any $x \in \Xs$, let $\partial f(x)$ denote the sub differential at $x$.
Access to the function $f$ is provided through a noisy first-order oracle.
At each step $t$, the learner queries the oracle with a point $x_t \in \Xs$ and receives $\hat{g}_t \in \R^d$ such that $\hat{g}_t = g_t - \xi_t$, where $g_t \in \partial f(x_t)$ and $\E[\xi_t \, | \, \xi_1, \dots, \xi_{t-1}] = 0$.

In the following, we use $\|\cdot\|$ to refer to a fixed arbitrary norm in $\R^d$.
Let $\psi \:\colon \: \R^d \rightarrow (-\infty, +\infty]$ be a convex function, and define $\dom(\psi) \coloneqq \{x\in\R^d \,\colon\, \psi(x) < +\infty\}$. For any $y \in \dom(\psi)$ at which $\psi$ is differentiable, the Bregman divergence at $y$ induced by $\psi$ is defined as
\begin{equation*}
    \breg{x}{y} = \psi(x) - \psi(y) - \langle x - y, \nabla \psi(y) \rangle \,,
\end{equation*}
for any $x \in \dom(\psi)$. 
For some $\lambda \geq 0$, $\psi$ is said to be $\lambda$-strongly convex with respect to $\|\cdot\|$ if $\psi(x) \geq \psi(y) + \langle x - y, g \rangle + (\lambda / 2) \|x-y\|^2$ for any $x,y \in \dom(\psi)$ and $g \in \partial\psi(y)$. This directly implies that $\breg{x}{y} \geq (\lambda / 2) \|x-y\|^2$ if $\psi$ is differentiable at $y$. 
To specify an instance of the mirror descent framework (see \Cref{alg:smd}), one needs to select a regularizer function $\psi$, which we will assume to satisfy the following:\footnote{For a set $S \subseteq \R$, $\inter(S)$ and $\partial S$ refer to its interior and boundary respectively.}
\begin{assumption} \label{assum:reg}
    The regularizer function $\psi \:\colon\: \R^d \rightarrow (-\infty, +\infty]$ is closed, differentiable on $\inter(\dom(\psi))$, $1$-strongly convex with respect to $\|\cdot\|$, and satisfies $\Xs \subseteq \dom(\psi)$ and $\inter (\dom(\psi)) \neq \{\}$. Moreover, it satisfies at least one of the following:
    (i) $\lim_{t \rightarrow \infty} \|\nabla \psi(x_t)\|_2 \rightarrow \infty$, for any sequence $(x_t)_t$ in $\inter(\dom(\psi))$ with $\lim_{t \rightarrow \infty} x_t \rightarrow x \in \partial \dom(\psi)$;
    (ii) $\Xs \subseteq \inter(\dom(\psi))$.
\end{assumption}
This is a standard assumption (see \citep{Beck2003} or \citep[Section 6.4]{Orabona2023}) that serves to insure that the iterates $(x_t)_t$ returned by the mirror descent algorithm are well-defined.
%and belong to $\inter(\dom(\psi))$---where $\psi$ is assumed to be differentiable and the Bregman divergence is well-defined---so long as $x_1 \in \inter(\dom(\psi))$. 
\begin{algorithm}[t] 
    \caption{Stochastic Mirror Descent} 
    \label{alg:smd}
    \begin{algorithmic} 
        \State \textbf{input:} regularizer $\psi$ satisfying \Cref{assum:reg}, non-increasing sequence of positive learning rates $(\eta_t)_t$ 
        \State \textbf{initialization:} choose $x_1 \in \inter(\dom(\psi))$
        \For{$t = 1, \ldots$}
            \State 
            output $x_t$ and receive $\hat{g}_t$ 
            \State set $x_{t+1} \gets \argmin_{x \in \Xs} \langle \hat{g}_t, x \rangle + \frac{1}{\eta_t} \breg{x}{x_t}$
        \EndFor
    \end{algorithmic}
\end{algorithm}
Denote by $\|\cdot\|_*$ the dual norm of $\|\cdot\|$, that is $\|\cdot\|_* \coloneqq \sup_{\|w\| \leq 1} \langle \cdot, w \rangle$. The following assumption implies that $f$ is Lipschitz with respect to $\|\cdot\|$.
\begin{assumption} \label{assum:lip}
    There exists a constant $G > 0$ such that for all $x \in \Xs$ and $g \in \partial f(x)$, $\|g\|_* \leq G$.
\end{assumption}

Let $f^* = \min_{x \in \Xs}f(x)$ and $x^* \in  \argmin_{x \in \Xs}f(x)$. For any $x \in \Xs$, we define the optimization error at $x$ as $f(x) - f^*$.
For some time horizon $T$, our goal in this work is to prove high probability bounds on the optimization error of the average iterate $\Bar{x}_T = (1/T)\sum_{t=1}^T x_t$ and the last iterate $x_T$ produced by \Cref{alg:smd}. 
Towards that end, we impose some restrictions on the noise vectors $(\xi_t)_t$. 
For what follows, let $\F_{t}$ be the sigma algebra generated by $(\xi_1, \dots, \xi_{t-1})$. Moreover, we will use $\E_t[\cdot]$ to denote $\E[\cdot \,|\, \F_{t-1}]$.
The following assumption provides a bound on the conditional second moment of $\|\xi_t\|_*$.
\begin{assumption} \label{assum:variance}
There exists a constant $\sigma > 0$ such that for every step $t \geq 1$,  it holds that $\E_t\bsb{\|\xi_t\|^2_*} \leq \sigma^2$.
\end{assumption}
This assumption is sufficient for proving in-expectation bounds, and tail bounds, but only of the order $1/(\delta \sqrt{T})$. We only use this as a base assumption when stating general facts.
Instead, we will instantiate our results under two different (stronger) assumptions on the noise terms $(\xi_t)_t$. The first assumption involves the class of sub-Weibull random variables \citep{Vladimirova2020,Kuchibhotla2022}, which generalizes the notions of sub-Gaussian and sub-exponential random variables. For $\theta >0 $ and $ \wnorm > 0$, we say that a random variable $X$ is sub-Weibull$(\theta, \wnorm)$ if it satisfies
$
    \E \bsb{ \exp \brb{\brb{|X| / \wnorm}^{1/\theta}} } \leq 2 \,.
$
At $\theta=1/2$, we recover the definition of a sub-Gaussian random variable, and at $\theta=1$, we recover that of a sub-exponential random variable \cite[Chapter 2]{Vershynin2018}. Via Markov's inequality, one can show that $X$ being sub-Weibull$(\theta,\wnorm)$  implies that for $t \geq 0$,
$
    P(|X| \geq t) \leq 2 \exp \brb{ - \lrb{t / \wnorm}^{1/\theta}} \,.
$
In this work, our focus is on the heavy-tailed regime where $\theta \geq 1$, though we also consider the canonical case of $\theta=1/2$ for comparison. In particular, we will consider the following assumption:
\begin{assumption} \label{assum:weibull}
    For some $\theta \geq 1$, there exists a constant $\wnorm > 0$ such that for every step $t \geq 1$, 
    $\|\xi_t\|_*$ is sub-Weibull($\theta,\wnorm$) conditioned on $\F_{t-1}$; that is,
    \begin{equation*}
        \E \Bsb{ \exp \Brb{\brb{\|\xi_t\|_* / \wnorm}^{1/\theta}} \, \big| \, \F_{t-1}} \leq 2 \,.
    \end{equation*}
\end{assumption}

Alternatively, we also consider the following assumption. 
\begin{assumption} \label{assum:poly}
    For some $p > 4$, there exists a constant $\wnorm > 0$ such that for every step $t \geq 1$,
    \begin{equation*}
        \E \Bsb{ \brb{\|\xi_t\|_* / \pnorm}^{p} \, \big| \, \F_{t-1}} \leq 1   \,.
    \end{equation*}
\end{assumption}
The above implies, via Markov's inequality, that $X$ satisfies the following polynomially decaying tail bound:
$
    P(|X| \geq t)\leq  (\pnorm / t)^p 
$
for any $t>0$. We only consider $p>4$ as the analyses in the sequel require studying the concentration properties of terms involving $\|\xi_t\|^2_*$.

\section{Average Iterate Analysis} \label{sec:avg}
When one's concern is studying the error of the average of the iterates $\Bar{x}_T$ at some time horizon $T$, a fairly standard analysis under \Cref{assum:lip,assum:reg,assum:variance} yields that
\begin{align*}
      f(\Bar{x}_T) - f^* 
      \leq  \frac{1}{\eta_T T} \bigg( \breg{x^*}{x_{1}} + \sum_{t=1}^T \eta_t^2 (G^2 + \sigma^2) + \underbrace{\sum_{t=1}^T \eta_t \ban{ \xi_t, x_t - x^*}}_{\coloneqq U} + \underbrace{\sum_{t=1}^T \eta_t^2 \brb{\|\xi_t\|_*^2 - \E_t \|\xi_t\|_*^2}}_{\coloneqq V} \bigg) \,.
\end{align*}
It is easy to verify that $\E U = \E V = 0$, which immediately yields a bound on the error in expectation. Proving a high-probability bound, on the other hand, requires controlling both terms in high probability. For $V$, this solely depends on the assumed statistical properties of $\|\xi_t\|_*$. Whereas for $U$, one also needs to control the terms $\|x_t - x^*\|$. This presents a major obstacle if one would like to avoid scaling with a bound on the diameter of the domain in terms of $\|\cdot\|$, which might not exist in some cases. In the recent work of \cite{Liu2023}, a more careful analysis distills this problem, roughly speaking, to bounding a term of the form
\begin{equation*}
    \summ_{t=1}^T w_t \eta_t \ban{ \xi_t, x_t - x^*} - v_t \|x_t-x^*\|^2,
\end{equation*}
where $(w_t)_t$ and $(v_t)_t$ are two carefully chosen sequences of weights. Assuming that the terms $\|\xi_t\|_*$ are conditionally sub-Gaussian, as done in \citep{Liu2023}, and applying the standard Chernoff method to bound this term in high probability, this refinement has the effect of normalizing the vectors $x_t-x^*$. Unfortunately, this ``white-box'' approach does not readily extend beyond the light-tailed case. For instance, if the noise terms are sub-exponential, it is not clear how to deal with the additional hurdle that the moment-generating function of $\ban{ \xi_t, x_t - x^*}$ is only bounded in a constrained range, whose diameter is inversely proportional to $\|x_t-x^*\|$.

In the more recent work of \cite{Nguyen2023}, a different weighting scheme is proposed for the purpose of analyzing a clipped version of SMD in a setting where it is only assumed that the $p$-th moment of the noise is bounded for $p \in (1,2]$. 
However, as presented, their analysis is still a ``white-box'' one, which leverages the properties of the clipped gradient estimate. 
In what follows, we demonstrate that a similar weighting scheme can be utilized in our setting to isolate the effect of the vectors $x_t-x^*$ in a ``black-box'' manner, independently of the assumed statistical properties of the noise. 
For $t \geq 1$, let 
\begin{equation} \label{def:maxD}
    D_t = \max\Bcb{ \gamma, \sqrt{\breg{x^*}{x_{1}}}, \dots, \sqrt{\breg{x^*}{x_{t}}} }, 
\end{equation}    
where $\gamma > 0$ is a constant that will be dictated by the analysis. 
Normalizing per-iterate quantities with $(D_t)_t$ is a natural choice as it is a non-decreasing sequence, predictable with respect to $(\F_t)_t$, and most notably, it holds that $\sqrt{2}D_t \geq \sqrt{2\breg{x^*}{x_{t}}} \geq \|x_t-x^*\|$. 
The following theorem provides a high probability bound on the error of the average of the iterates without requiring an upper bound on the diameter of the domain, as long as one can control the tails of two martingales essentially depending only on the noise.
\begin{theorem} \label{thm:avg-general}
Let $Y_1, Y_2 \,:\,(0,1) \times [0,\infty)^T \rightarrow (0,\infty)$ be two functions such that for any $\delta \in (0,1)$,
    \begin{equation*}
       P\bbrb{\max_{s \leq T}\summ_{t=1}^s \eta_t \Ban{ \xi_t, \frac{x_t - x^*}{\sqrt{2}D_t}} > Y_1\brb{\delta,(\eta_t)_{t=1}^T}} \leq \delta %%%
       \hspace{0.43em}\text{and} \hspace{0.43em}  
       P\bbrb{\summ_{t=1}^T \eta_t^2 \brb{\|\xi_t\|_*^2 - \E_t \|\xi_t\|_*^2} > Y_2\brb{\delta,(\eta_t)_{t=1}^T}} \leq \delta \,,
    \end{equation*}
    %and
    %\begin{equation*}
    %   P\bbrb{\summ_{t=1}^T \eta_t^2 \brb{\|\xi_t\|_*^2 - \E_t \|\xi_t\|_*^2} > Y_2\brb{\delta,(\eta_t)_{t=1}^T}} \leq \delta \,,
    %\end{equation*}
    where $D_t$ is as defined in \eqref{def:maxD} with $\gamma$ chosen as $\sqrt{Y_2\brb{\delta/2,(\eta_t)_{t=1}^T} + \sum_{t=1}^T  \eta_t^2 \brb{G^2 + \sigma^2}}$.
    Then, under \Cref{assum:lip,assum:reg,assum:variance}, \Cref{alg:smd} satisfies the following with probability at least $1-\delta$:
    \begin{align*}
        f(\Bar{x}_T) - f^* \leq \frac{3}{\eta_T T} \bigg( \breg{x^*}{x_{1}} + \sum_{t=1}^T  \eta_t^2 \brb{G^2 + \sigma^2} 
        + 2 Y_1\brb{\delta/2,(\eta_t)_{t=1}^T}^2 + Y_2\brb{\delta/2,(\eta_t)_{t=1}^T} \bigg) \,. 
    \end{align*}
\end{theorem}
\begin{proof}
    \Cref{lem:smd-weighted-iterates} in \Cref{app: basic} with $z=x^*$ and $w_t=1/D_t$ yields that for any $s\in[T]$
    \begin{align*}
        \frac{\breg{x^*}{x_{s+1}}}{D_s} + \sum_{t=1}^s \frac{\eta_t}{D_t}  (f(x_t) - f^*) \leq \frac{\breg{x^*}{x_{1}}}{D_1} + \sum_{t=1}^s \frac{ \eta_t^2}{2 D_t}\|\hat{g}_t\|_*^2 + \sum_{t=1}^s \eta_t \Ban{ \xi_t, \frac{x_t - x^*}{D_t}}\,.
    \end{align*}
    For brevity, define $d_t = \sqrt{\breg{x^*}{x_{t}}}$. Taking the previous inequality further, we have that
    \begin{align} \label{eq:unbounded-domain-main-relation}
        \frac{d_{s+1}^2}{D_s} +  \sum_{t=1}^s \frac{\eta_t}{D_t}  (f(x_t) - f^*) \leq d_1 + \gamma \lor \lrb{ \frac{1}{\gamma} \sum_{t=1}^T \frac{ \eta_t^2}{2}\|\hat{g}_t\|_*^2}  
        + \sqrt{2} \bbrb{\max_{s \leq T}\sum_{t=1}^s \eta_t \Ban{ \xi_t, \frac{x_t - x^*}{\sqrt{2}D_t}}}_+,
    \end{align}
    where $x \lor y = \max\{x,y\}$ and $x_+ = \max\{0,x\}$. Define $B_T$ as the right-hand side of the last inequality. Consequently, we have that $d_{s+1}^2 \leq D_s B_T$, and thanks to the non-negativity of the last term in the right-hand side of \eqref{eq:unbounded-domain-main-relation}, we have that $D_1=\max\{d_1,\gamma\} \leq B_T$. Moreover, if $D_s \leq B_T$ for some $s \in [T]$, then \[D_{s+1} = \max\{D_s,d_{s+1}\} \leq \max \{D_s,\sqrt{D_s B_T}\} \leq B_T \,. \]
    Thus, via induction, $D_s \leq B_T$ for all $s \in [T]$. Since $(\eta_t)_t$ and $(D_t)_t$ are non-increasing and non-decreasing respectively, we can conclude from \eqref{eq:unbounded-domain-main-relation} and the convexity of $f$ that
    \begin{align} \label{eq:unbounded-domain-brief-bound}
        f(\Bar{x}_T) - f^* \leq \frac{1}{T}\sum_{t=1}^T (f(x_t) - f^*) \leq \frac{D_T B_T}{\eta_T T} \leq \frac{B_T^2}{\eta_T T} \,.
    \end{align} 
    Utilizing \Cref{assum:lip,assum:variance}, we have that
    \begin{align*}
        \sum_{t=1}^T \frac{ \eta_t^2}{2}\|\hat{g}_t\|_*^2 
        &= \sum_{t=1}^T \frac{ \eta_t^2}{2}\|g_t - \xi_t\|_*^2  \\
        &\leq \sum_{t=1}^T  \eta_t^2 \brb{\|g_t\|_*^2 + \|\xi_t\|_*^2} \\
        &= \sum_{t=1}^T  \eta_t^2 \brb{\|g_t\|_*^2 + \E_t \|\xi_t\|_*^2} + \sum_{t=1}^T  \eta_t^2 \brb{\|\xi_t\|_*^2 - \E_t \|\xi_t\|_*^2} \\
        &\leq \sum_{t=1}^T  \eta_t^2 \brb{G^2 + \sigma^2} + \sum_{t=1}^T  \eta_t^2 \brb{\|\xi_t\|_*^2 - \E_t \|\xi_t\|_*^2} \,.
    \end{align*}
    Combining this with the assumed tail bounds and plugging in the value of $\gamma$ yields that
    \begin{align*}
        B_T  \leq d_1 + \sqrt{Y_2\brb{\delta/2,(\eta_t)_{t=1}^T} + \summ_{t=1}^T  \eta_t^2 \brb{G^2 + \sigma^2}}
        + \sqrt{2} Y_1\brb{\delta/2,(\eta_t)_{t=1}^T} \,,
    \end{align*}
    with probability at least $1-\delta$, which, combined with \eqref{eq:unbounded-domain-brief-bound}, allows us to conclude the proof after simple calculations.
\end{proof}
\Cref{thm:avg-general} provides a modular bound, turning which into a concrete convergence rate requires applying suitable martingale concentration results, depending on the adopted noise model. Starting with the sub-Weibull case, \Cref{lem:subw-martingale-conc} in \Cref{app:concentration}, a more versatile version of a result in Proposition~11 in \citep{Madden2020}, provides a maximal concentration inequality for martingales with conditionally sub-Weibull increments. Utilizing this results leads to the following corollary.
\begin{restatable}{corollary}{weibullavg} \label{cor:weibullavg}
    For any $\delta \in (0,1)$ and $\eta > 0$, \Cref{alg:smd}, under \Cref{assum:reg,assum:lip,assum:weibull},  satisfies the following with probability at least $1-\delta$.
    \begin{enumerate} [(i)] 
    \item If $\eta_t = \eta$, 
    \begin{align*}
        f(\Bar{x}_T) - f^* \leq \frac{C}{T} \bigg( \frac{\breg{x^*}{x_{1}}}{\eta} + \eta \brb{G^2 + \wnorm^2 \log(e / \delta)}T 
        + \eta \phi^2 \log^{2\theta} (e T/ \delta ) \bigg)
    \end{align*}
    \item If $\eta_t = \frac{\eta}{\sqrt{t}}$, 
    \begin{align*}
        f(\Bar{x}_T) - f^* \leq \frac{C\log(eT)}{\sqrt{T}} \bbrb{ \frac{\breg{x^*}{x_{1}}}{\eta} + \eta \lrb{G^2 + \phi^2 \log^{2 \theta}(e / \delta) }} 
    \end{align*}
    \end{enumerate}
    where $C$ is a constant depending only on $\theta$.
    
\end{restatable}
A proof is provided in \Cref{app:avg}. Firstly, we remark that these bounds can also be shown to hold in the sub-Gaussian setting (with $\theta=1/2$), where they recover the corresponding results in \citep{Liu2023}.
Also notice that, regardless of $\theta$, as $\wnorm$ goes to zero, we recover the standard bounds for the deterministic setting.
In the case when $\eta_t = \eta$ (the known time horizon setting), the bound exhibits what we will refer to as a \emph{two-regime} behaviour. To better illustrate this, consider that an optimal tuning of $\eta$ yields a bound of order
\begin{align*}
    \sqrt{\breg{x^*}{x_{1}}} \lrb{ \sqrt{ \frac{G^2 + \wnorm^2 \log(e / \delta) }{T}} + \frac{\phi \log^{\theta} (e T/ \delta ) }{T} } \,.
\end{align*}
The first term in the brackets is the standard sub-Gaussian rate, while the second depends on the assumed shape of the noise. 
The key observation here is that 
%(for a fixed confidence level)
as the horizon grows longer, the sub-Gaussian term will eventually come to dominate, masking the heavy-tailed behaviour of the noise. 
This turning point depends, most importantly, on the required confidence level $1-\delta$ and the shape parameter $\theta$. 
It is also noteworthy that the second term is primarily the contribution of the noise at a single step, a phenomenon inherited from the Freedman-style concentration inequalities on which this result is based.

In the case when $\eta_t = \eta/\sqrt{t}$ (the anytime setting), the bound in \Cref{cor:weibullavg} is akin, in form, to results presented in \citep{Madden2020,Li2022} in the non-convex setting under different assumptions.
However, we avoid the extra dependence on $\log^{2\theta} (T)$ featured in these works thanks to the general form of \Cref{lem:subw-martingale-conc}, which allows one to take advantage of the fact that the learning rate schedule is imbalanced to retain the same dependence on $T$ as in the light-tailed case. 
On the other hand, this imbalance also means that for both martingales featured in \Cref{thm:avg-general}, the effect of the noise in the beginning (when $\eta_t$ is large) is, in a sense, comparable to that of the whole sequence. 
On the surface, this explains why the bound we presented in the anytime case does not exhibit the two-regime behaviour enjoyed by the first bound. 
The deeper cause is that the analysis relies on controlling the maximum of the terms $\|x_t - x^*\|$ in high probability, which seems to naturally result in the dominance of the heavy-tailed regime. 
In fact, it is not difficult (see \Cref{app:anytime} for the proof of a stronger statement) to show that under the assumption that $\max_t \sqrt{\breg{x^*}{x_{t}}} \leq D$, one can obtain a bound of order 
%\begin{align*} \label{eq:wei-avg-anytime-two-regime}
%    \frac{1}{\sqrt{T}} \lrb{ \frac{D^2}{\eta}  + \eta \bbrb{G^2 + \phi^2 \bbrb{ \log(e / \delta) + \frac{\log^{2 \theta}(e T/ \delta)} {\sqrt{T}}}  }} \,. 
%\end{align*}
\begin{align*}
    \frac{1}{\sqrt{T}} \lrb{ \frac{D^2}{\eta} + \eta \bbrb{ G^2 + \phi^2 \bbrb{ \log(e / \delta) + \frac{\log^{2\theta} \lrb{ {e /  \delta} }}{\sqrt{T}} + \frac{\log^{2\theta} \brb{e T/ \delta }}{T} }}} \,.
\end{align*}
Even if one cannot generally tune $\eta$ optimally (as $T$ is unknown), the message is that as $T$ grows, the bound approaches its sub-Gaussian counterpart. Deriving a similar guarantee 
%(for general convex and Lipschitz functions)
without assuming a bound on the diameter remains an interesting problem.

Under \Cref{assum:poly}, one can use Fuk-Nagaev type concentration inequalities (see, e.g., \cite{Rio2017}) to control the tails of the martingales in question. Doing so, we arrive at the following corollary, whose proof is provided in \Cref{app:avg}. 
\begin{restatable}{corollary}{polyavg} \label{cor:polyavg}
    For any $\delta \in (0,1)$ and $\eta > 0$, \Cref{alg:smd}, under \Cref{assum:reg,assum:lip,assum:poly},  satisfies the following with probability at least $1-\delta$.
    \begin{enumerate} [(i)]
    \item If $\eta_t = \eta$, 
    \begin{align*} 
        f(\Bar{x}_T) - f^* \leq \frac{C}{T} \bbrb{ \frac{\breg{x^*}{x_{1}}}{\eta} + \eta \brb{G^2 + \phi^2 \log(e / \delta)}T + \eta \pnorm^2 \brb{T / \delta}^{2/p}} 
    \end{align*}
    \item If $\eta_t = \frac{\eta}{\sqrt{t}}$, 
    \begin{align*}
        f(\Bar{x}_T) - f^* \leq \frac{C \log(eT)}{\sqrt{T}} \bbrb{ \frac{\breg{x^*}{x_{1}}}{\eta} + \eta \Brb{G^2 + \pnorm^2 \brb{1 / \delta}^{2/p}} } 
    \end{align*}
    \end{enumerate}
    where $C$ is a constant depending only on $p$.
\end{restatable}
The bounds are analogous to the sub-Weibull case, except that the terms accounting for the heavy tailed behaviour feature a polynomial (instead of logarithmic) dependence on $1/\delta$. A suitable tuning of $\eta$ in the first case leads to a bound of order
\begin{align*}
    \sqrt{\breg{x^*}{x_{1}}} \lrb{ \sqrt{ \frac{G^2 + \wnorm^2 \log(e / \delta) }{T}} + \frac{\phi \brb{1 / \delta}^{1/p} }{T^{1-1/p}} } \,.
\end{align*}
Notice that $T^{1-1/p} > T^{3/4}$; hence,  
also in this case, the sub-Gaussian term can dominate if the horizon is long enough.
A similar bound was reported in \cite{Lou2022} for the particular setting of a linear regression problem with the squared loss,\footnote{In their setting, it was only assumed that $p>2$.} where the two-regime behaviour of the bound was also highlighted.

In the anytime setting, similar to the sub-Weibull case, the bound retains the same dependence on $T$ as in the sub-Gaussian case, but only exhibits heavy-tailed behaviour. Analogously to %\eqref{eq:wei-avg-anytime-two-regime},
the sub-Weibull case,
when $\max_t \sqrt{\breg{x^*}{x_{t}}} \leq D$,
one can prove (see \Cref{app:anytime}) a bound of order
\begin{align*} 
%\label{eq:poly-avg-anytime-two-regime}
    \frac{1}{\sqrt{T}} \lrb{ \frac{D^2}{\eta}  + \eta \bbrb{G^2 + \phi^2 \bbrb{ \log(e / \delta) + \frac{(1/\delta)^{2/p}} {\sqrt{T}}}  }} 
\end{align*}
The question of deriving a similar bound (for general convex and Lipschitz functions)
without assuming a bound on the diameter is more pressing in this case, as the steeper polynomial dependence on $1/\delta$ would otherwise call for the use of truncation.

\section{Last Iterate Analysis} \label{sec:last}
Focusing on the anytime case, a typical last iterate analysis in the non-smooth setting \citep{Shamir13,Harvey2019} starts with a bound of the following form:\footnote{Proofs for the results presented in this section can be found in \Cref{app:last}.}
\begin{restatable} {lemma} {lastiterateunrolled} \label{lem:last-iterate-unrolled}
    \Cref{alg:smd} with $\eta_t = \frac{\eta}{\sqrt{t}}$ for some constant $\eta > 0$ satisfies
    \begin{align*}
        f(x_T) - f^* \leq \frac{2}{T} \sum_{t=\ceil{T/2}}^T \brb{f(x_t) - f^*} + \sum_{t=\ceil{T/2}}^T \ban{\xi_t,  w_t} + \frac{\eta}{\sqrt{2T}} \sum_{t=\ceil{T/2}}^T \rho_t \|\hat{g}_t\|^2_* + \frac{\sqrt{2}}{\eta \sqrt{T}} \sum_{t=\ceil{T/2}}^T z_t\,,
    \end{align*}
    where, for $j < T$, $\alpha_j = \frac{1}{(T-j)(T-j+1)}$, and for any time-step $t \geq \ceil{T/2}$, 
    \begin{equation*}
        w_t = \sum_{j=\ceil{T/2}}^{t\land (T-1)} \alpha_j (x_t - x_{j}) \,, \quad z_t = \sum_{j=\ceil{T/2}}^{t\land (T-1)} \alpha_j \breg{x_{j}}{x_t} \,, %%%
        \quad \text{and} \quad \rho_t = \sum_{j=\ceil{T/2}}^{t\land (T-1)} \alpha_j \,.
    \end{equation*}
\end{restatable}
The first term in the bound can be dealt with using the techniques of the previous section, the third term appears in the analysis of the previous section (albeit with different weights) and can be handled similarly, while the last term is usually handled using a uniform bound on the divergence terms, though this is not necessary as we will see. 
It is not difficult then to show that these three terms decay at a rate of at most $\log(T)/\sqrt{T}$ with high probability.
The main obstacle in the way of proving a tail bound for the error is showing that the second (martingale) term enjoys a similar rate.
%in high probability.
Naively bounding the norms of the vectors $w_t$ using a diameter bound is not sufficient. Instead, one needs to exploit the peculiar structure of this term. 

For the following, define the martingale sequence $(Q_s)_{s=\ceil{T/2}}^T$ where $Q_s = \sum_{t=\ceil{T/2}}^s \lan{\xi_t,  w_t}$, and denote by $\lan{Q}_s$ its total conditional variance (TCV), i.e., $\lan{Q}_s = \sum_{t=\ceil{T/2}}^s \E_t \lan{\xi_t,  w_t}^2$. Via the convexity of $\|\cdot\|^2$ and the fact that $\|x_t-x_j\|^2 \leq 2 \breg{x_{j}}{x_t}$, it holds that $\|w_t\|^2 \leq 2 \rho_t z_t$. Thus, under \Cref{assum:variance}, one can verify that
$\lan{Q}_s \leq 2 \sigma^2 \sum_{t=\ceil{T/2}}^s \rho_t z_t$.
The key observation of \cite{Harvey2019} is that this sum can be bounded with an affine function of the martingale itself.
Via a generalized version of Freedman's inequality, the authors exploit the resulting fact that $\lan{Q}_T$ is upper bounded with a suitable affine function of $Q_T$ to arrive at the desired tail bound. 
This inequality, however, is once again specific to the sub-Gaussian noise setting, beyond which one usually needs finer control on the individual $w_t$ terms, as argued in the previous section.
Hence, once again, we seek an approach through which we can disentangle the vectors $w_t$ from the noise terms $\xi_t$.
The following lemma provides a starting point by showing that $z^* \coloneqq \max_{\ceil{T/2} \leq s \leq T} z_s$ can itself be related to $(Q_s)_s$.
\begin{restatable}{lemma}{maxzlemma} \label{lem:max-z}
    In the same setting as \Cref{lem:last-iterate-unrolled}, it holds that
    \begin{align*}
        z^* \leq \frac{6\sqrt{2} \eta}{T\sqrt{T}} \summ_{t=\ceil{T/2}}^{T} \brb{f(x_t) - f^*} + \frac{3\sqrt{2} \eta}{\sqrt{T}} Q_{n^*} + \frac{3\eta^2}{T} \summ_{t=\ceil{T/2}}^{T} \rho_t \|\hat{g}_t\|^2_* \,,
    \end{align*}
    where $n^* = \min \bcb{ n \:\colon\: n \in \argmax_{\ceil{T/2} \leq s \leq T}Q_s }$.
\end{restatable}
A nice implication of this lemma is that the last term in the bound of \Cref{lem:last-iterate-unrolled} can be related to the preceding terms. 
However, at this point, this lemma does not provide a tight (high probability) bound on the $z_t$ (or $\|w_t\|$) terms due to the dependence on $Q_{n^*}$.
Thus, techniques relying on such a bound, like the averaging scheme of the previous section or extensions of the concentration result of \cite{Harvey2019} to sub-Weibull random variables in \citep[Proposition 11]{Madden2020},\footnote{The latter would actually require an almost sure bound.} are not easily utilizable. Instead, the real advantage of this lemma is that it allows one to relate not only the TCV but also the total quadratic variation (TQV) of $Q_T$, given by $[Q]_T = \sum_{t=\ceil{T/2}}^T \lan{\xi_t,  w_t}^2$, back to the martingale itself through $z^*$:
\begin{restatable}{lemma}{tqvtcv}\label{lem:tqv-tcv-bound}
In the same setting as \Cref{lem:last-iterate-unrolled}, it holds under \Cref{assum:variance} that
\begin{align*}
    \lan{Q}_{T} + [Q]_{T} \leq 4 \sigma^2 z^* \log(4T) + 2 z^* \summ_{t=\ceil{T/2}}^{T} \rho_t \brb{\|\xi_t\|^2_* - \E_{t}\|\xi_t\|^2_*}\,.
\end{align*}
\end{restatable}
The sum in the second term occurs also when bounding the third term in the bound of \Cref{lem:last-iterate-unrolled}, and has been encountered in the average iterate analysis. Notice that, trivially, the left hand side of \Cref{lem:tqv-tcv-bound} is also a bound for the sum of the TCV and TQV at any step, particularly at $n^*$. Being able to bound this sum allows one to derive powerful concentration results with few assumptions. In the next proposition, we extend one such result, Theorem 2.1 in \citep{Bercu2008}, in the spirit of Theorem 3.3 in \citep{Harvey2019}.
\begin{restatable}{proposition}{chickenegg} \label{thm:chicken-egg}
Let $(M_t)_{t=0}^n$ be a square integrable martingale adapted to filtration $(\F_t)_{t=0}^n$ with $M_0 = 0$. Then, for all $x, \beta > 0$ and $\alpha \geq 0$,
\begin{align*}
    P\lrb{\bigcup_{t=1}^n \bcb{M_t \geq x \:\text{and}\: \lan{M}_t + [M]_t \leq \alpha M_t + \beta } } \leq \exp\bbrb{-\min\bbcb{\frac{x^2}{8\beta}, \frac{x}{6\alpha}}} \,. 
\end{align*}
\end{restatable}
Utilizing this tool, together with the preceding lemmas, we arrive at the following general bound for the last iterate.
\begin{restatable}{theorem} {thmlast} \label{thm:last}
    Let $\Xi_1, \Xi_2: (0,1) \rightarrow (0, \infty)$ be two functions such that for any $\delta \in (0,1)$,
    \begin{equation*}
       P\lrb{\frac{1}{\sqrt{T}} \summ_{t=1}^T \brb{f(x_t) - f^*} > \Xi_1(\delta)} \leq \delta %%%
       \quad \text{and} \quad P\lrb{\summ_{t=\ceil{T/2}}^{T} \rho_t \brb{\|\xi_t\|^2_* - \E_{t}\|\xi_t\|^2_*} > \Xi_2(\delta)} \leq \delta \,.
    \end{equation*}
    Then, under \Cref{assum:lip,assum:reg,assum:variance}, \Cref{alg:smd} with $\eta_t = \frac{\eta}{\sqrt{t}}$ satisfies the following with probability at least $1-\delta$:
    \begin{align*}
        f(x_T) - f^* \leq \frac{35}{\sqrt{T}} \bigg( 2 \Xi_1(\delta/3) + \sqrt{2}\eta G^2 \log(4T) 
        + 9\sqrt{2}\eta \Brb{\Xi_2(\delta /3) + 2 \sigma^2 \log(4T)} \log (3/\delta) \bigg) \,.
    \end{align*}    
\end{restatable}

To obtain a concrete bound, one needs a tail bound for the error of the average iterate and a similar bound for a by-now-familiar martingale term. The following corollary provides concrete bounds for our two noise models.
\begin{restatable} {corollary} {corlast} \label{cor:last}
    For any $\delta \in (0,1)$ and $\eta > 0$, \Cref{alg:smd} with $\eta_t = \frac{\eta}{\sqrt{t}}$ satisfies the following with probability at least $1-\delta$, where $C_1$ and $C_2$ are constant depending solely on, respectively, $\theta$ and $p$.
    \begin{enumerate} [(i)]
    \item Under \Cref{assum:reg,assum:lip,assum:weibull}, 
    \begin{align*}
        f(x_T) - f^* \leq \frac{C_1\log(eT)}{\sqrt{T}} \bbrb{ \frac{\breg{x^*}{x_{1}}}{\eta} + \eta \lrb{G^2 + \phi^2 \log^{2 \theta + 1}(e / \delta) }} 
    \end{align*}
    \item Under \Cref{assum:reg,assum:lip,assum:poly}, 
    \begin{align*}
         f(x_T) - f^* \leq \frac{C_2 \log(eT)}{\sqrt{T}} \bbrb{ \frac{\breg{x^*}{x_{1}}}{\eta}  
         + \eta \Brb{G^2 + \pnorm^2 \brb{1 / \delta}^{2/p} \log(e / \delta)} } 
    \end{align*}
    \end{enumerate}    
\end{restatable}
\begin{figure*}[ht]
\centering
\begin{subfigure}[b]{0.34\textwidth}
\includegraphics[width=\textwidth]{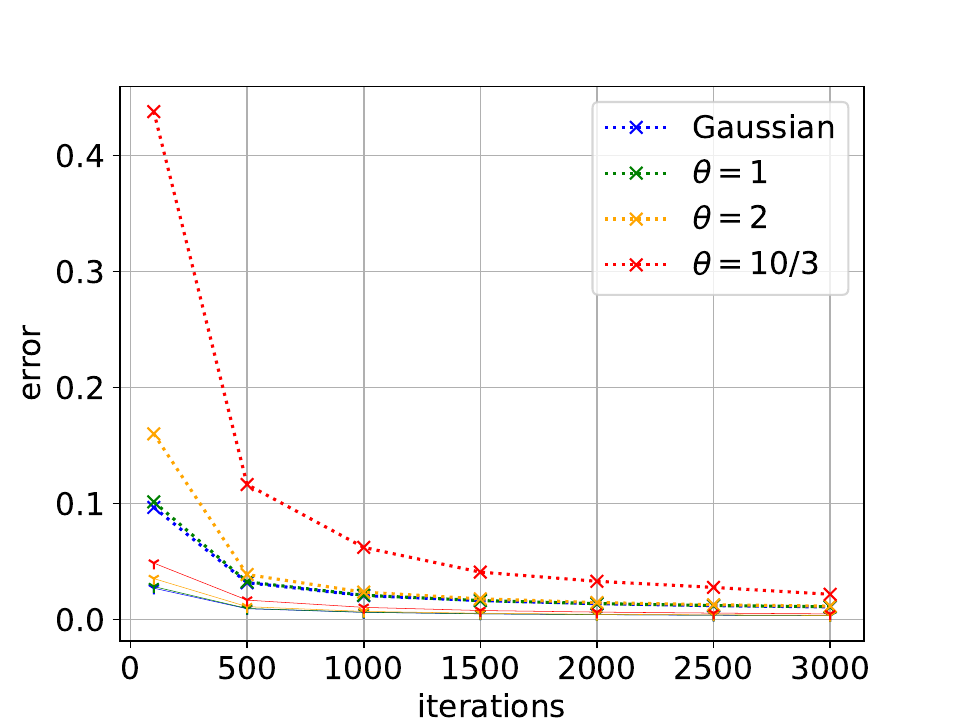}
\caption{}
\label{fig:a}
\end{subfigure}
\begin{subfigure}[b]{0.34\textwidth}
\includegraphics[width=\textwidth]{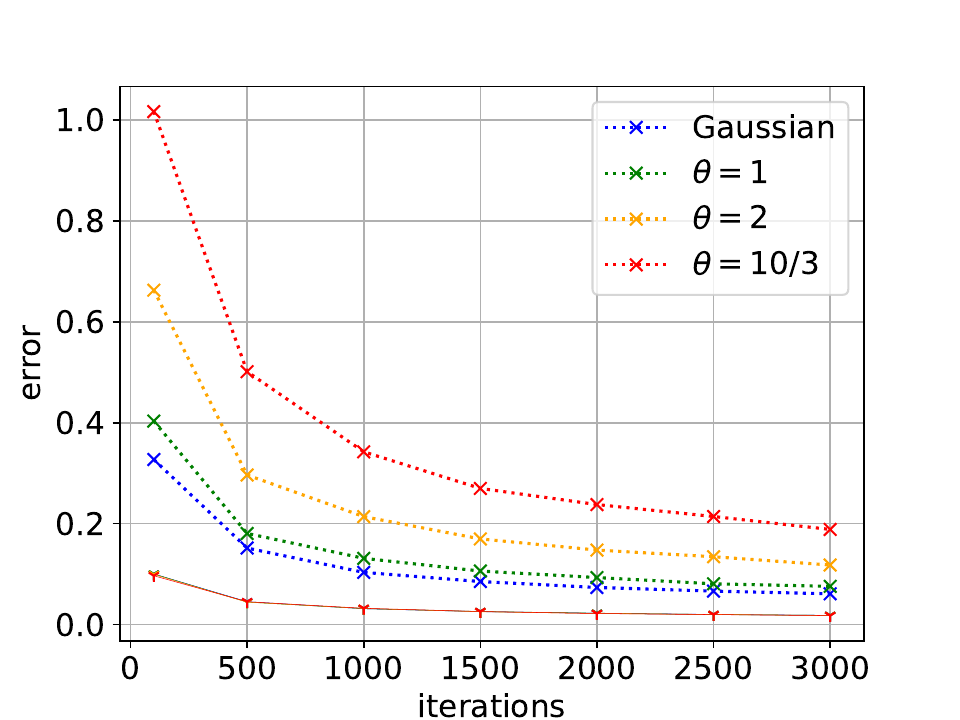}
\caption{}
\label{fig:b}
\end{subfigure}
\begin{subfigure}[b]{0.34\textwidth}
\includegraphics[width=\textwidth]{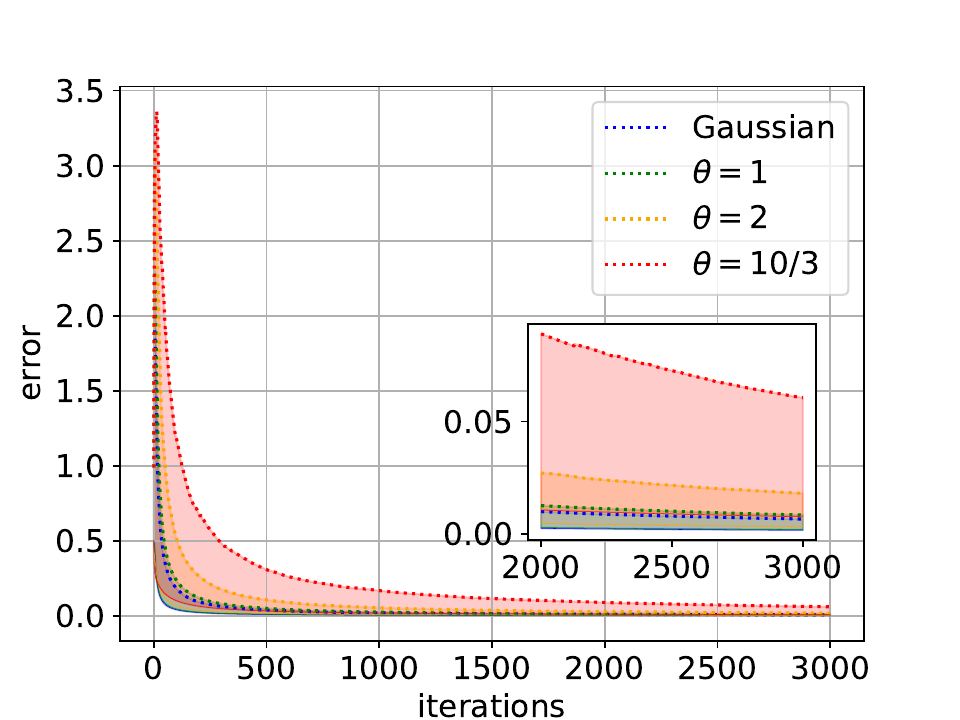}
\caption{}
\label{fig:c}
\end{subfigure}
\begin{subfigure}[b]{0.34\textwidth}
\includegraphics[width=\textwidth]{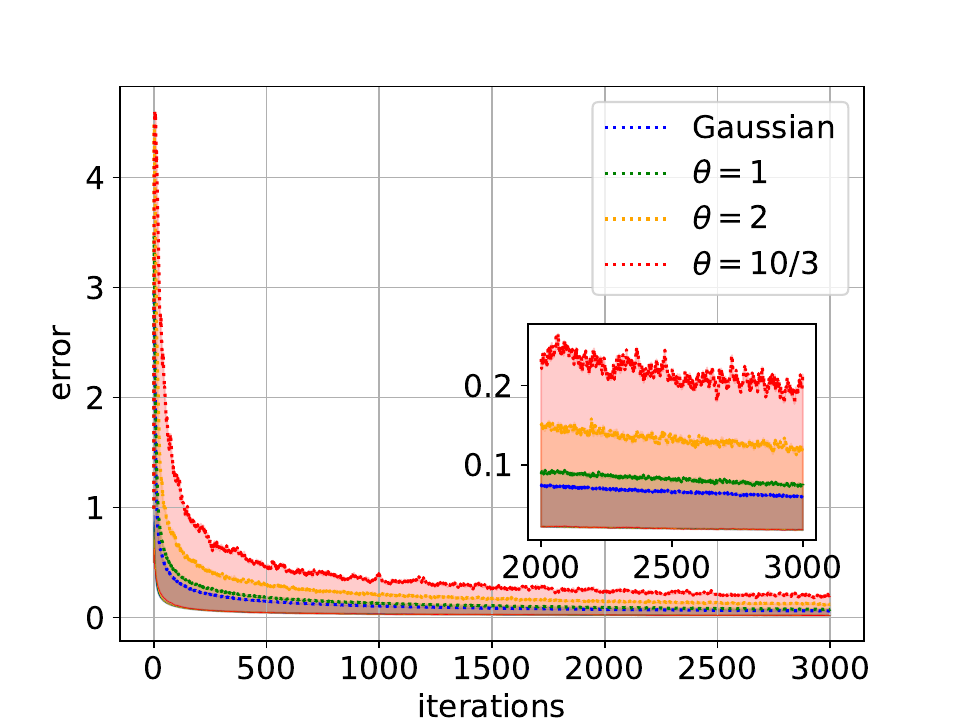}
\caption{}
\label{fig:d}
\end{subfigure}
\caption{The performance of the average iterate and the last iterate are reported in the figures on the left and the right columns respectively. Solid lines show the average of the error across runs and dotted lines show its $99$-percentile. The top row refers to the case of $\eta=1/T$, while the bottom row refers to the case of $\eta=1/\sqrt{t}$. 
The zoomed plots highlight the performance in the last 1k iterations.}
\label{fig:anytime}
\end{figure*}
Firstly, these bounds retain the same decay rate in $T$ as that in the deterministic case, whose bounds are recovered as the noise vanishes. Compared to their counterparts in the average case, both bounds contain an extra $\log(1/\delta)$ factor, an artifact of the general disentanglement technique we adopt. While this means that the exact sub-Gaussian rate is not recovered, this factor is arguably negligible for heavier noise. Although we focused on the anytime learning rates $\eta/\sqrt{t}$, similar results can be straightforwardly verified to hold when using a constant learning rate. Interestingly, for either schedule, the bounds obtainable from this analysis do not assume the two-regime form. The main obstacle for this is encountered as early as the fairly standard \Cref{lem:last-iterate-unrolled}, and is manifested in the third term therein. 
This term leads to the dominance of the heavy-tailed regime, primarily through the contribution of the noise in the final iterates, where $\rho_t$ is $\Theta(1)$.
Beyond the standard step-size choices, extending the analysis of the scheme proposed by \cite{Jain2021} to heavy-tailed noise is an interesting problem.

\section{Experiments} \label{sec:exp}
We present two experiments comparing the performance of the average of the iterates with that of the last iterate when using \Cref{alg:smd} to minimize $f(x) = |x|$ over $\mathbb{R}$ 
with $\psi(x) = 1/2 \|x\|_2^2$ (i.e., classical SGD). 
For the noise, we consider the Gaussian distribution with variance $1$ and three different Weibull distributions with $\theta=1, 2, 10/3$ respectively.
For a fair comparison, the Weibull distributions are scaled to have unit variance. In each experiment, we run the algorithm for $3$k iterations, repeated $20$k times. We report the average and the $99$-percentile of the optimization errors. In the first experiment, we use $1/\sqrt{T}$ as a fixed step-size and run the algorithm for seven values of $T$ ranging from $100$ to $3$k, reporting only the errors at the end of each run. The results for the average iterate and the last iterate are reported in figures (a) and (b) respectively. While in both plots the average error is almost the same across noise levels (due to the normalization), the $99$-percentile curves show a significant difference in behaviour between the two plots. In particular, for the average iterate, the curves for the heavy-tailed noise distributions approach the Gaussian level as the horizon grows, as predicted by the two-regime bounds.
Whereas for the last iterate, the different noise levels exhibit a clear separation for all values of $T$, indicating higher sensitivity to heavy-tailed noise.

In the second experiment, we set $\eta_t=1/\sqrt{t}$ and report the evolution of the error through the $3$k iterations for the average iterate and the last iterate in figures (c) and (d) respectively. We observe once again that the $99$-percentile curves for the last iterate remain well separated across the entire run. On the other hand, in the average iterate case, the very small scale of the $y$-axis in the zoomed plot and the steeper slope of the $99$-percentile curves (with respect to the Gaussian one) seem to hint towards a two-regime behaviour in the anytime case as well.

\bibliographystyle{plainnat}
\bibliography{refs}

\newpage
\appendix
\section{Basic Results for Stochastic Mirror Descent} 
The following lemma is a standard result for mirror descent; see, for example, \cite[Lemma 6.9]{Orabona2023}.
\label{app: basic}
\begin{lemma} \label{lem:smd-one-step}
    For any $z \in \Xs$, the iterates $(x_t)_t$ output by \Cref{alg:smd} satisfy
    \begin{equation*}
        f(x_t) - f(z) \leq \frac{1}{\eta_t} \breg{z}{x_t} - \frac{1}{\eta_t} \breg{z}{x_{t+1}} + \langle \xi_t, x_t - z \rangle + \frac{\eta_t}{2} \|\hat{g}_t\|_*^2 \,.
    \end{equation*}
\end{lemma}

\begin{proof}
    Since $x_{t+1}$ is the minimizer of the convex function $\Phi_t(x) = \langle \hat{g}_t, x \rangle + \frac{1}{\eta_t} \breg{x}{x_t}$ in $\Xs$, it satisfies that for any $z \in \Xs$,
    \begin{equation} \label{eq:smd-optimality}
        \langle \hat{g}_t + (1/\eta_t) \nabla \psi(x_{t+1}) - (1/\eta_t) \nabla \psi(x_t), x_{t+1} - z \rangle = \langle \nabla \Phi_t(x_{t+1}), x_{t+1} - z  \rangle \leq 0 \,. 
    \end{equation}
    Hence,
    \begin{align*}
        \eta_t \langle \hat{g}_t, x_{t} - z \rangle &= \eta_t \langle \hat{g}_t, x_{t} - x_{t+1} \rangle + \eta_t \langle \hat{g}_t, x_{t+1} - z \rangle \\
        &= \eta_t \langle \hat{g}_t, x_{t} - x_{t+1} \rangle + \langle \eta_t \hat{g}_t +  \nabla \psi(x_{t+1}) -  \nabla \psi(x_t), x_{t+1} - z \rangle +  \langle \nabla \psi(x_{t}) -  \nabla \psi(x_{t+1}), x_{t+1} - z \rangle \\
        &\stackrel{(a)}{\leq} \eta_t \langle \hat{g}_t, x_{t} - x_{t+1} \rangle + \langle \nabla \psi(x_{t}) -  \nabla \psi(x_{t+1}), x_{t+1} - z \rangle \\
        &\stackrel{(b)}{=} \eta_t \langle \hat{g}_t, x_{t} - x_{t+1} \rangle + \breg{z}{x_t} - \breg{z}{x_{t+1}} - \breg{x_{t+1}}{x_t} \\
        &\stackrel{(c)}{\leq} \breg{z}{x_t} - \breg{z}{x_{t+1}} + \eta_t \|\hat{g}_t\|_* \|x_t-x_{t+1}\| - \frac{1}{2} \|x_t-x_{t+1}\|^2 \\
        &\stackrel{(d)}{\leq} \breg{z}{x_t} - \breg{z}{x_{t+1}} + \frac{\eta^2_t}{2} \|\hat{g}_t\|_*^2 \,,
    \end{align*}
    where (a) holds via \eqref{eq:smd-optimality}, (b) holds via \citep[Lemma 4.1]{Beck2003}, (c) holds by the $1$-strong convexity of $\psi$ and the fact that (by the definition of the dual norm) $\|\hat{g}_t\|_* = \sup_{x \in \R^d \setminus \{\bm{0}\}} \langle \hat{g}_t, x / \|x\| \rangle $, and (d) holds since $ax - (1/2)x^2 \leq (1/2)a^2$ for $x, a \in \R$. After dividing by $\eta_t$, the lemma follows using that $\hat{g}_t = g_t - \xi_t$ and the fact that $\langle g_t, x_{t} - z \rangle \geq f(x_t) - f(z)$ as $g_t \in \partial f(x_t)$.
\end{proof}

\begin{lemma} \label{lem:smd-weighted-iterates}
    For any $z \in \Xs$ and any non-increasing sequence of positive weights $(w_t)_t$, \Cref{alg:smd} satisfies that for any $s \geq 1$,
    \begin{equation*}
        w_s \breg{z}{x_{s+1}} + \sum_{t=1}^s w_t \eta_t (f(x_t) - f(z)) \leq  w_1 \breg{z}{x_{1}} + \sum_{t=1}^s \frac{w_t \eta_t^2}{2}\|\hat{g}_t\|_*^2 + \sum_{t=1}^s w_t \eta_t \ban{ \xi_t, x_t - z} \,.
    \end{equation*}
\end{lemma}
\begin{proof}
    Since both $\eta_t$ and $w_t$ are non-negative, it follows from \Cref{lem:smd-one-step} that
    \begin{equation*}
        w_t \eta_t \brb{f(x_t) - f(z)} \leq w_t \breg{z}{x_t} - w_t \breg{z}{x_{t+1}} + w_t \eta_t \langle \xi_t, x_t - z \rangle + \frac{w_t \eta_t^2}{2} \|\hat{g}_t\|_*^2\,.
    \end{equation*}
    Using that $(w_t)_t$ is a non-increasing sequence, we have that
    \begin{align*}
        \sum_{t=1}^s w_t \brb{\breg{z}{x_t} - \breg{z}{x_{t+1}}} &= w_1 \breg{z}{x_{1}} - w_s \breg{z}{x_{s+1}} + \sum_{t=2}^s \breg{z}{x_{t}} (w_t - w_{t-1}) \\
        &\leq  w_1 \breg{z}{x_{1}} - w_s \breg{z}{x_{s+1}} \,,
    \end{align*}
    which entails that
    \begin{equation*}
        \sum_{t=1}^s w_t \eta_t (f(x_t) - f(z)) \leq  w_1 \breg{z}{x_{1}} - w_s \breg{z}{x_{s+1}} + \sum_{t=1}^s w_t \eta_t \ban{ \xi_t, x_t - z} + \sum_{t=1}^s \frac{w_t \eta_t^2}{2}\|\hat{g}_t\|_*^2 \,.
    \end{equation*}
\end{proof}

\begin{lemma} \label{lem:smd-iterate-comparison}
    Let $j$ and $r$ be two time indices such that $j \leq r$, and define $\tilde{\eta}_t = \frac{1}{ \eta_t} -  \frac{1}{ \eta_{t-1}}$ and $\tilde{\eta}_1 = \frac{1}{\eta_1}$. Then, \Cref{alg:smd} satisfies that
    \begin{align*}
        \frac{1}{ \eta_r} \breg{x_j}{x_{r+1}} + \sum_{t=j}^r \brb{f(x_t) - f(x_j)} \leq \sum_{t=j}^r \langle \xi_t, x_t - x_j \rangle + \frac{1}{2} \sum_{t=j}^r \eta_t \|\hat{g}_t\|_*^2 + \sum_{t=j}^r \tilde{\eta}_t \breg{x_j}{x_t}  \,.
    \end{align*}
\end{lemma}
\begin{proof}
    For $t \geq j$, \Cref{lem:smd-one-step} implies that 
    \begin{align*}
         f(x_t) - f(x_j) \leq \frac{1}{\eta_t} \breg{x_j}{x_t} - \frac{1}{\eta_t} \breg{x_j}{x_{t+1}} + \langle \xi_t, x_t - x_j \rangle + \frac{\eta_t}{2} \|\hat{g}_t\|_*^2  \,.
    \end{align*}
    Notice that,
    \begin{align*}
        \sum_{t=j}^r \bbrb{ \frac{1}{ \eta_t}\breg{x_j}{x_t} - \frac{1}{ \eta_t} \breg{x_j}{x_{t+1}}} &= \frac{1}{ \eta_j} \breg{x_j}{x_{j}} - \frac{1}{ \eta_r} \breg{x_j}{x_{r+1}} + \sum_{t={j+1}}^r \bbrb{ \frac{1}{ \eta_t} -  \frac{1}{ \eta_{t-1}}} \breg{x_j}{x_t} \\
        &= -\frac{1}{ \eta_r} \breg{x_j}{x_{r+1}} + \sum_{t=j}^r \tilde{\eta}_t \breg{x_j}{x_t}\,,
    \end{align*}
    where we have used that $\breg{x_j}{x_j} = 0$. Thus, we conclude that
    \begin{align*}
         \sum_{t=j}^r \brb{f(x_t) - f(x_j)} \leq \sum_{t=j}^r \langle \xi_t, x_t - x_j \rangle + \frac{1}{2} \sum_{t=j}^r \eta_t \|\hat{g}_t\|_*^2- \frac{1}{ \eta_r} \breg{x_j}{x_{r+1}} + \sum_{t=j}^r \tilde{\eta}_t \breg{x_j}{x_t} \,.
    \end{align*}
\end{proof}

\section{Proofs of Section \ref{sec:avg}} \label{app:avg}
Before proving \Cref{cor:weibullavg,cor:polyavg}, we state two lemmas specializing \Cref{lem:subw-martingale-conc,lem:fuk-nagaev-martingale-conc} in \Cref{app:concentration} to the two martingales we encounter when analyzing SMD.
\begin{lemma} \label{lem:shortcut-subw-conc}
    Let $(\omega_t)_{t=1}^T$ be a sequence of positive (deterministic) weights with $\omega_*$ denoting their maximum. Additionally, let $(u_t)_{t=1}^T$ be a sequence of vectors in $\R^d$ such that $u_t$ is $\F_{t-1}$-measurable and $\|u_t\| \leq 1$. Then, under \Cref{assum:weibull}, the following holds for any $\delta \in (0,1)$ and $s \geq 0$.
    \begin{enumerate}[(i)]
        \item 
        \begin{equation*}
             P\lrb{\max_{k\in[T]}\sum_{t=1}^k \omega_t \ban{ \xi_t, u_t}  \geq  \phi \sqrt{C_1  \sum_{t=1}^T \omega_t^2 \log(2 / \delta)} + 4 \phi \omega_* C_2 \log^{\theta} \lrb{ \frac{2 e \sum_{t=1}^T \omega_t^s}{\omega_*^s \delta} } } \leq \delta  \,,
        \end{equation*}
        where $C_1 = 2^{3 \theta + 1} \Gamma(3\theta + 1)$ and $C_2 = \max \bcb{ 1 ,  (s \theta - s)^{\theta-1}}$.
        \item
        \begin{equation*}
        P\lrb{\max_{k\in[T]}\sum_{t=1}^k \omega_t \brb{\|\xi_t\|_*^2 - \E_t \|\xi_t\|_*^2} \geq   C_3 \phi^2 \sqrt{C_1  \sum_{t=1}^T \omega_t^2 \log(2 / \delta)} + 4 C_2 C_3 \phi^2 \omega_* \log^{2\theta} \lrb{ \frac{2 e \sum_{t=1}^T \omega_t^s}{\omega_*^s \delta} }  } \leq \delta  \,,
        \end{equation*}
        where $C_1 = 2^{6 \theta + 1} \Gamma(6\theta + 1)$, $C_2 = \max \bcb{ 1 ,  (2s \theta - s)^{2\theta-1}}$, and $C_3 = 2^{2\theta+1}\Gamma(2\theta+1) / \ln^{2\theta}(2)$.
    \end{enumerate}
\end{lemma}
\begin{proof}
    \emph{(i)} Since $\|u_t\| \leq 1$, the definition of the dual norm implies that $ |\omega_t \ban{ \xi_t, u_t}| \leq \omega_t \|u_t\| \|\xi_t\|_* \leq \omega_t \|\xi_t\|_*$, yielding that $ \omega_t \ban{ \xi_t, u_t}$ is sub-Weibull($\theta, \omega_t \wnorm$) conditioned on $\F_{t-1}$. The result then follows from \Cref{lem:subw-martingale-conc}(ii).
    
    \emph{(ii)} Using the definition of the sub-Weibull property, one can easily verify that if a random variable X is sub-Weibull($\theta, \wnorm$); then, $X^2$ is sub-Weibull($2\theta, \wnorm^2$). Using this along with \Cref{lem:weibull-centering} yields that $\omega_t \brb{\|\xi_t\|_*^2 - \E_t \|\xi_t\|_*^2}$ is sub-Weibull($2\theta, c_\theta \omega_t \wnorm^2$) conditioned on $\F_{t-1}$, where $c_\theta =  2^{2\theta+1}\Gamma(2\theta+1) / \ln^{2\theta}(2)$. Hence, the result once more follows from \Cref{lem:subw-martingale-conc}(ii).
\end{proof}

\begin{lemma} \label{lem:shortcut-poly-conc}
    Let $(\omega_t)_{t=1}^T$ be a sequence of positive (deterministic) weights with $\omega_*$ denoting their maximum. Additionally, let $(u_t)_{t=1}^T$ be a sequence of vectors in $\R^d$ such that $u_t$ is $\F_{t-1}$-measurable and $\|u_t\| \leq 1$. Then, under \Cref{assum:poly}, the following holds for any $\delta \in (0,1)$.
    \begin{enumerate}[(i)]
        \item 
        \begin{equation*}
            P\lrb{ \max_{k\in[T]}\sum_{t=1}^k \omega_t \ban{ \xi_t, u_t} > \phi \sqrt{2 \sum_{t=1}^T \omega_t^2 \log(1/\delta)} + \brb{2+(p/3)} \phi \lrb{\sum_{t=1}^T \omega_t^p / \delta}^{1/p} } \leq \delta \,.
        \end{equation*}
        \item
        \begin{equation*}
            P\lrb{\max_{k\in[T]}\sum_{t=1}^k \omega_t \brb{\|\xi_t\|_*^2 - \E_t \|\xi_t\|_*^2} >  2 \pnorm^2 \sqrt{2 \sum_{t=1}^T \omega_t^2 \log(1/\delta)} + 2\brb{2+(p/6)} \pnorm^2 \lrb{\sum_{t=1}^T \omega_t^{p/2} / \delta}^{2/p}} \leq \delta \,.
        \end{equation*}
    \end{enumerate}
\end{lemma}
\begin{proof}
    \emph{(i)}  From the definition of the dual norm and the fact that $\|u_t\| \leq 1$, we have that
    \begin{equation*}
        \E \Bsb{ \babs{ \omega_t \ban{  \xi_t, u_t}}^{p} \, \big| \, \F_{t-1}} \leq \omega_t^p \E \Bsb{ \|u_t\|^p \|\xi_t\|_*^{p} \, \big| \, \F_{t-1}} \leq \omega_t^p \E \Bsb{  \|\xi_t\|_*^{p} \, \big| \, \F_{t-1}} \leq \brb{\omega_t \pnorm}^p   \,,
    \end{equation*}
    where the last inequality follows form \Cref{assum:poly}. The result then follows from \Cref{lem:fuk-nagaev-martingale-conc}. 
    
    \emph{(ii)} On the other hand,
    \begin{equation*}
        \E \Bsb{ \babs{ \omega_t \brb{\|\xi_t\|_*^2 - \E_t \|\xi_t\|_*^2}}^{p/2} \, \big| \, \F_{t-1}} \leq 2^{p/2} \omega_t^{p/2} \E \Bsb{ \|\xi_t\|_*^{p} \, \big| \, \F_{t-1}} \leq \brb{2 \omega_t \pnorm^2}^{p/2}   \,,
    \end{equation*}
    where the first inequality follows from \Cref{lem:centering-norms} and the second follows from \Cref{assum:poly}. Consequently, the result follows once more from \Cref{lem:fuk-nagaev-martingale-conc}.
\end{proof}

\subsection{Proof of Corollary \ref{cor:weibullavg}}
\weibullavg*
\begin{proof}
    For $t \in [T]$, let $u_t = (x_t - x^*)/(\sqrt{2}D_t)$, while for $k \in [T]$, we define 
    \[W_k = \sum_{t=1}^k \eta_t \ban{ \xi_t, u_t} \qquad \text{and} \qquad V_k = \sum_{t=1}^k \eta_t^2 \brb{\|\xi_t\|_*^2 - \E_t \|\xi_t\|_*^2} \,.\]
    As argued before, it holds that $\sqrt{2}D_t \geq \|x_t-x^*\|$, implying that $\|u_t\| \leq 1$.
    For what follows, we will use $C,C_1,C_2,\dots$ to denote positive constants---depending only on $\theta$---whose values may change between steps.
    
    \paragraph{Case \emph{(i)}: $\eta_t = \eta$} ~\\
    Starting with $(W_k)$, we invoke \Cref{lem:shortcut-subw-conc}(i) with $s=0$ and $\omega_t=\eta$ obtaining that
    \begin{equation*}
        P\bbrb{\max_{k \in [T]} W_k \geq C_1 \eta \wnorm \sqrt{ T \log(2 / \delta)} + C_2 \eta \wnorm  \log^{\theta} \brb{2 e T/ \delta }} \leq \delta \,.
    \end{equation*}
    For $(V_k)$, we invoke \Cref{lem:shortcut-subw-conc}(ii) with $s=0$ and $\omega_t = \eta^2$ to get that
    \begin{equation*}
        P\bbrb{ \max_{k \in [T]} V_k \geq C_1 \eta^2 \phi^2 \sqrt{T \log(2 / \delta)} + C_2 \eta^2 \phi^2  \log^{2\theta} \brb{2 e T/ \delta }} \leq \delta \,.
    \end{equation*}
    With these tail bounds, \Cref{thm:avg-general} implies that
    \begin{align*}
        \frac{\eta T}{3} \brb{f(\Bar{x}_T) - f^*} &\leq \breg{x^*}{x_{1}} + \eta^2 \brb{G^2 + C_1 \wnorm^2} T + C_2 \eta^2 \wnorm^2 \Brb{ T \log(4 / \delta) +  \log^{2\theta} \brb{4 e T/ \delta }} \\
        ~
        &\hspace{20em}+ C_3 \eta^2 \phi^2 \Brb{ \sqrt{  T \log(4 / \delta)} +  \log^{2\theta} \brb{4 e T/ \delta }} \\
        ~
        &\leq \breg{x^*}{x_{1}} + \eta^2 \brb{G^2 + C_1 \wnorm^2} T + C_2 \eta^2 \wnorm^2 \Brb{ T \log(4 / \delta) +  \log^{2\theta} \brb{4 e T/ \delta }} \,,
    \end{align*}
    where we have used the fact that \Cref{assum:weibull} implies \Cref{assum:variance} with $\sigma^2 = 2 \Gamma(2\theta + 1) \wnorm^2$ thanks to \Cref{lem:subw-moments}. Subsequently, we have that
    \begin{align*}
        f(\Bar{x}_T) - f^* \leq \frac{C}{T} \bbrb{ \frac{1}{\eta} \breg{x^*}{x_{1}} + \eta \brb{G^2 + \wnorm^2 \log(e / \delta)}T + \eta \phi^2 \log^{2\theta} (e T/ \delta )} \,.
    \end{align*}

    \paragraph{Case \emph{(ii)}: $\eta_t = \eta/\sqrt{t}$} ~\\
    For $(W_k)$, we use \Cref{lem:shortcut-subw-conc}(i) with $s=3$ and $\omega_t=\eta/\sqrt{t}$, while for $(V_k)$, we use the \Cref{lem:shortcut-subw-conc}(ii) with $s=2$ and $\omega_t = \eta^2 / t$ yielding that 
    \begin{equation*}
             P\lrb{\max_{k\in[T]} W_k  \geq C_1 \eta \phi \sqrt{\sum_{t=1}^T (1/t) \log(2 / \delta)} + C_2 \eta \phi  \log^{\theta} \lrb{ {2 e \sum_{t=1}^T (1/t)^{3/2} / \delta} } } \leq \delta  \,,
        \end{equation*}
    and
    \begin{equation*}
        P\lrb{\max_{k\in[T]} V_k \geq   C_1 \eta^2 \phi^2 \sqrt{\sum_{t=1}^T (1/t)^2 \log(2 / \delta)} + C_2 \eta^2 \phi^2 \log^{2\theta} \lrb{ {2 e \sum_{t=1}^T (1/t)^2 /  \delta} }  } \leq \delta  \,.
        \end{equation*}
    Combining this with the facts that
    \begin{equation*}
        \sum_{t=1}^T \frac{1}{t} \leq \log(eT)\,, \qquad \sum_{t=1}^T \frac{1}{t^{3/2}} \leq 3\,, \qquad \text{and} \qquad \sum_{t=1}^T \frac{1}{t^2} \leq 2 \,,
    \end{equation*}
    implies via \Cref{thm:avg-general} that
    \begin{align*}
        \frac{\eta \sqrt{T}}{3} \brb{f(\Bar{x}_T) - f^*} &\leq \breg{x^*}{x_{1}} + \eta^2 \brb{G^2 + C_1 \phi^2} \log(eT) + C_2 \eta^2 \wnorm^2 \Brb{ \log(eT) \log(4 / \delta) +  \log^{2\theta} \brb{12 e/ \delta }} \\
        ~
        &\hspace{20em}+ C_3 \eta^2 \phi^2 \Brb{ \sqrt{  \log(4 / \delta)} +  \log^{2\theta} \brb{8 e/ \delta }} \\
        ~
        &\leq \breg{x^*}{x_{1}} + \eta^2 \brb{G^2 + C_1 \phi^2} \log(eT) + C_2 \eta^2 \wnorm^2 \Brb{ \log(eT) \log(4 / \delta) +  \log^{2\theta} \brb{12 e/ \delta }} \\
        ~
        &\leq \breg{x^*}{x_{1}} + \eta^2 \brb{G^2 + C_1 \phi^2} \log(eT) + C_2 \eta^2 \wnorm^2  \log(eT) \log^{2\theta} \brb{12 e/ \delta } \,,
    \end{align*}
    where we have again used \Cref{lem:subw-moments} to bound $\E_t \|\xi_t\|_*^2$ in terms of $\wnorm^2$ (in place of $\sigma^2$) under \Cref{assum:weibull}.
    Hence, we conclude that
    \begin{align*}
        f(\Bar{x}_T) - f^* \leq \frac{C\log(eT)}{\sqrt{T}} \bbrb{ \frac{1}{\eta} \breg{x^*}{x_{1}} + \eta \lrb{G^2 + \phi^2 \log^{2 \theta}(e / \delta) }} \,.
    \end{align*}
\end{proof}

\subsection{Proof of Corollary \ref{cor:polyavg}}
\polyavg*
\begin{proof}
    Similar to the proof of \Cref{cor:weibullavg}, we define $u_t = (x_t - x^*)/(\sqrt{2}D_t)$ (which satisfies $\|u_t\| \leq 1$), and consider once again the two martingale terms 
    \[W_k = \sum_{t=1}^k \eta_t \ban{ \xi_t, u_t} \qquad \text{and} \qquad V_k = \sum_{t=1}^k \eta_t^2 \brb{\|\xi_t\|_*^2 - \E_t \|\xi_t\|_*^2} \,.\]
    For $(W_k)$, we use \Cref{lem:shortcut-poly-conc}(i) with $\omega_t=\eta_t$, while for $(V_k)$, we use the \Cref{lem:shortcut-poly-conc}(ii) with $\omega_t = \eta_t^2$ yielding that
    \begin{equation*}
        P\lrb{ \max_{k \in [T]} W_k >  \pnorm \sqrt{2 \sum_{t=1}^T \eta_t^2 \log(1/\delta)} + \brb{2+(p/3)} \pnorm \lrb{\sum_{t=1}^T \eta_t^p / \delta}^{1/p}} \leq \delta \,,
    \end{equation*}
    and
    \begin{equation*}
        P\lrb{ \max_{k \in [T]} V_k >  2 \pnorm^2 \sqrt{2 \sum_{t=1}^T  \eta_t^4 \log(1/\delta)} + 2 \brb{2+(p/6)} \pnorm^2 \lrb{\sum_{t=1}^T \eta_t^{p} / \delta}^{2/p}} \leq \delta \,.
    \end{equation*}
    
    For what follows, we will use $C$ to denote a positive constant---depending only on $p$---whose value may change between steps.
    
    \paragraph{Case \emph{(i)}: $\eta_t = \eta$} ~\\
    \Cref{thm:avg-general} with the tail bounds above yields that
    \begin{align*}
        \frac{\eta T}{3} \brb{f(\Bar{x}_T) - f^*} &\leq \breg{x^*}{x_{1}} + \eta^2 \brb{G^2 + \pnorm^2} T + 4 \eta^2 \pnorm^2 \Brb{ 2 T \log(2/\delta) + \brb{2+(p/3)}^2  \brb{2T / \delta}^{2/p}} \\
        ~
        &\hspace{19em}+ 2 \eta^2 \pnorm^2  \Brb{ \sqrt{2 T \log(2/\delta)} +  \brb{2+(p/6)} \brb{2T / \delta}^{2/p}} \\
        ~
        &\leq \breg{x^*}{x_{1}} + \eta^2 \brb{G^2 + \phi^2} T + 6 \eta^2 \pnorm^2 \Brb{ 2 T \log(2/\delta) + \brb{2+(p/3)}^2  \brb{2T / \delta}^{2/p}} \,,
    \end{align*}
    where we have used the fact that \Cref{assum:poly} implies \Cref{assum:variance} with $\sigma^2 = \pnorm^2$. Subsequently, we have that
    \begin{align*}
        f(\Bar{x}_T) - f^* &\leq \frac{3}{T} \bbrb{ \frac{1}{\eta} \breg{x^*}{x_{1}} + \eta \Brb{G^2 + \phi^2 \brb{ 1 + 12 \log(2 / \delta) }}T + 6 \eta \pnorm^2 \brb{2+(p/3)}^2  \brb{2T / \delta}^{2/p}} \\
        &\leq \frac{C}{T} \bbrb{ \frac{1}{\eta} \breg{x^*}{x_{1}} + \eta \brb{G^2 + \phi^2 \log(e / \delta)}T + \eta \pnorm^2 \brb{T / \delta}^{2/p}} \,.
    \end{align*}

    \paragraph{Case \emph{(ii)}: $\eta_t = \eta/\sqrt{t}$} ~\\
    Using that
    \begin{equation*}
        \sum_{t=1}^T \eta_t^2 = \eta^2 \sum_{t=1}^T \frac{1}{t} \leq \eta^2 \log(eT) \,, \qquad \sum_{t=1}^T \eta_t^4 = \eta^4 \sum_{t=1}^T \frac{1}{t^2} \leq 2 \eta^4 \,, \qquad  \text{and} \qquad \sum_{t=1}^T \eta_t^p = \eta^p \sum_{t=1}^T \frac{1}{t^{p/2}} \leq 2 \eta^p
    \end{equation*}
    as $p > 4$ and $t \geq 1$, \Cref{thm:avg-general} implies that
    \begin{align*}
        \frac{\eta \sqrt{T}}{3} \brb{f(\Bar{x}_T) - f^*} &\leq \breg{x^*}{x_{1}} + \eta^2 \brb{G^2 + \pnorm^2} \log(eT) + 4 \eta^2 \pnorm^2 \Brb{ 2 \log(eT) \log(2/\delta) + \brb{2+(p/3)}^2 \brb{4 / \delta}^{2/p}} \\
        ~
        &\hspace{20em}+ 4 \eta^2 \pnorm^2 \Brb{ \sqrt{ \log(2/\delta)} + \brb{2+(p/6)} \brb{4 / \delta}^{2/p}} \\
        ~
        &\leq \breg{x^*}{x_{1}} + \eta^2 \brb{G^2 + \pnorm^2} \log(eT) + 8 \eta^2 \pnorm^2 \Brb{ 2 \log(2/\delta) + \brb{2+(p/3)}^2 \brb{4 / \delta}^{2/p}} \log(eT) \\
        ~
        &\leq \breg{x^*}{x_{1}} + \eta^2 \brb{G^2 + \pnorm^2} \log(eT) + 8 \eta^2 \pnorm^2 \Brb{ p (2/\delta)^{2/p} + \brb{2+(p/3)}^2 \brb{4 / \delta}^{2/p}} \log(eT)  \,,
    \end{align*}
    where we have used that $\log(2/\delta) \leq (p/2) (2/\delta)^{2/p}$, and once again used $\pnorm^2$ in place of $\sigma^2$ by virtue of \Cref{assum:poly}.
    Hence, we conclude that
    \begin{align*}
        f(\Bar{x}_T) - f^* &\leq \frac{3\log(eT)}{\sqrt{T}} \bbrb{ \frac{1}{\eta} \breg{x^*}{x_{1}} + \eta \Brb{G^2 + \pnorm^2 + 16 \brb{2+p}^2 \pnorm^2 \brb{4 / \delta}^{2/p}} } \\
        &\leq \frac{C \log(eT)}{\sqrt{T}} \bbrb{ \frac{1}{\eta} \breg{x^*}{x_{1}} + \eta \Brb{G^2 + \pnorm^2 \brb{1 / \delta}^{2/p}} } \,.
    \end{align*}
\end{proof}

\section{Bounds for the Average Iterate Under a Bounded Domain Assumption} \label{app:anytime}
In this section, we consider again the case when $\eta_t = \eta / \sqrt{t}$ and prove, under a bounded domain assumption,
error bounds for the average iterate that assume a two-regime form. We start with following standard error bound.
\begin{lemma} \label{lem:bounded-avg}
    Assume that there exits $D>0$ such that $\sqrt{\breg{x}{y}} \leq D$ for any $(x,y) \in \dom(\psi) \times \inter(\dom(\psi))$. Then, under \Cref{assum:reg,assum:lip,assum:variance}, \Cref{alg:smd} satisfies
    \begin{equation*}
        f(\bar{x}_T) - f^* \leq \frac{1}{T} \lrb{ \frac{D^2}{\eta_T} +  \sum_{t=1}^T \eta_t (G^2 + \sigma^2) + \sum_{t=1}^T \langle \xi_t, x_t - x^* \rangle + \sum_{t=1}^T \eta_t \brb{\|\xi_t\|_*^2 - \E_t\|\xi_t\|_*^2} } \,.
    \end{equation*}
\end{lemma}
\begin{proof}
    \Cref{lem:smd-one-step} with $z=x^*$ yields that
    \begin{equation*}
        f(x_t) - f^* \leq \frac{1}{\eta_t} \breg{x^*}{x_t} - \frac{1}{\eta_t} \breg{x^*}{x_{t+1}} + \langle \xi_t, x_t - x^* \rangle + \frac{\eta_t}{2} \|\hat{g}_t\|_*^2 \,.
    \end{equation*}
    Summing this inequality we obtain that
    \begin{align*}
        \sum_{t=1}^T \brb{f(x_t) - f^*} &\leq \sum_{t=1}^T \frac{1}{\eta_t} \breg{x^*}{x_t} - \sum_{t=1}^T \frac{1}{\eta_t} \breg{x^*}{x_{t+1}} + \sum_{t=1}^T \langle \xi_t, x_t - x^* \rangle + 
        \frac{1}{2} \sum_{t=1}^T \eta_t \|\hat{g}_t\|_*^2 \\
        &= \frac{1}{\eta_1} \breg{x^*}{x_1} - \frac{1}{\eta_T} \breg{x^*}{x_{T+1}} + \sum_{t=2}^T \breg{x^*}{x_t} \lrb{\frac{1}{\eta_t} - \frac{1}{\eta_{t-1}}} 
        \\&\hspace{25em}+ \sum_{t=1}^T \langle \xi_t, x_t - x^* \rangle + 
        \frac{1}{2} \sum_{t=1}^T \eta_t \|\hat{g}_t\|_*^2 \\
        &\leq \frac{D^2}{\eta_1} + D^2 \sum_{t=2}^T \lrb{\frac{1}{\eta_t} - \frac{1}{\eta_{t-1}}} + \sum_{t=1}^T \langle \xi_t, x_t - x^* \rangle + 
        \frac{1}{2} \sum_{t=1}^T \eta_t \|\hat{g}_t\|_*^2 \\
        &= \frac{D^2}{\eta_T} + \sum_{t=1}^T \langle \xi_t, x_t - x^* \rangle + 
        \frac{1}{2} \sum_{t=1}^T \eta_t \|\hat{g}_t\|_*^2 \,.
    \end{align*}
    The required result then follows using the fact that $f(\bar{x}_T) - f^* \leq \frac{1}{T} \sum_{t=1}^T \brb{f(x_t) - f^*}$ and that 
    \[
        \|\hat{g}_t\|_*^2 = \|g_t - \xi_t\|_*^2 
        \leq 2\brb{\|g_t\|_*^2 + \|\xi_t\|_*^2} 
        =  2\brb{\|g_t\|_*^2 + \E_t\|\xi_t\|_*^2} + 2 \brb{\|\xi_t\|_*^2 - \E_t\|\xi_t\|_*^2}
        \leq 2\brb{G^2 + \sigma^2} + 2 \brb{\|\xi_t\|_*^2 - \E_t\|\xi_t\|_*^2} \,,
    \]
    where we used \Cref{assum:lip,assum:variance} in the last step.
\end{proof}
We then state the two following corollaries specializing the result of the last lemma under \Cref{assum:weibull,assum:poly} respectively.

\begin{corollary}\label{cor:weibullavg-bounded}
    Assume that there exits $D>0$ such that $\sqrt{\breg{x}{y}} \leq D$ for any $(x,y) \in \dom(\psi) \times \inter(\dom(\psi))$.
    Then, for any $\delta \in (0,1)$ and $\eta > 0$, \Cref{alg:smd} with $\eta_t = \frac{\eta}{\sqrt{t}}$ satisfies, under \Cref{assum:reg,assum:lip,assum:weibull}, that with probability at least $1-\delta$,
    \begin{align*}
        f(\Bar{x}_T) - f^* &\leq \frac{C_1}{\sqrt{T}} \bbrb{ \frac{D^2}{\eta} + \eta G^2 + \wnorm D \bbrb{ \sqrt{\log(e / \delta)} + \frac{\log^{\theta} \brb{e T/ \delta }}{\sqrt{T}}} 
        + \eta \phi^2 \bbrb{ 1 + \sqrt{\frac{\log(eT) \log(e / \delta)}{T}} + \frac{\log^{2\theta} \lrb{ {e /  \delta} }}{\sqrt{T}} }} \\
        &\leq \frac{C_2}{\sqrt{T}} \bbrb{ \frac{D^2}{\eta} + \eta G^2  
        + \eta \phi^2 \bbrb{ \log(e / \delta) + \frac{\log^{2\theta} \lrb{ {e /  \delta} }}{\sqrt{T}} + \frac{\log^{2\theta} \brb{e T/ \delta }}{T} }} \,,
    \end{align*}
    where $C_1$ and $C_2$ are constants depending only on $\theta$.
\end{corollary}
\begin{proof}
     For $t \in [T]$, let $u_t = (x_t - x^*)/(\sqrt{2}D)$, while for $k \in [T]$, we define 
    \[W_k = \sum_{t=1}^k \ban{ \xi_t, u_t} \qquad \text{and} \qquad V_k = \sum_{t=1}^k \eta_t \brb{\|\xi_t\|_*^2 - \E_t \|\xi_t\|_*^2} \,.\]
    Since $\sqrt{2}D \geq \sqrt{2\breg{x^*}{x_t}} \geq \|x_t-x^*\|$, it holds that $\|u_t\| \leq 1$.
    For what follows, we will use $C,C_1,C_2,\dots$ to denote positive constants---depending only on $\theta$---whose values may change between steps. For the first martingale $(W_k)$, we invoke \Cref{lem:shortcut-subw-conc}(i) with $s=0$ and $\omega_t=1$ obtaining that
    \begin{equation*}
        P\bbrb{\max_{k \in [T]} W_k \geq C_1 \wnorm \sqrt{ T \log(2 / \delta)} + C_2 \wnorm  \log^{\theta} \brb{2 e T/ \delta }} \leq \delta \,.
    \end{equation*}
    while for $(V_k)$, we use the \Cref{lem:shortcut-subw-conc}(ii) with $s=3$ and $\omega_t = \eta / \sqrt{t}$ yielding that 
    \begin{equation*}
        P\lrb{\max_{k\in[T]} V_k \geq   C_1 \eta \phi^2 \sqrt{\sum_{t=1}^T (1/t) \log(2 / \delta)} + C_2 \eta \phi^2 \log^{2\theta} \lrb{ {2 e \sum_{t=1}^T (1/t)^{3/2} /  \delta} }  } \leq \delta  \,.
    \end{equation*}
    Since $\sum_{t=1}^T \frac{1}{\sqrt{t}} \leq 2\sqrt{T}$, $\sum_{t=1}^T (1/t) \leq \log(eT)$, and $\sum_{t=1}^T (1/t)^{3/2} \leq 3$, \Cref{lem:bounded-avg} implies via a union bound that with probability at least $1-\delta$,
    \begin{align*}
        T\brb{f(\bar{x}_T) - f^*} &\leq   \frac{D^2 \sqrt{T}}{\eta} +  C_1 \eta (G^2 + \wnorm^2) \sqrt{T} + C_2 \wnorm D \Brb{ \sqrt{ T \log(4 / \delta)} + \log^{\theta} \brb{4 e T/ \delta }} 
        \\&\hspace{20em}+ C_3 \eta \phi^2 \Brb{ \sqrt{\log(eT) \log(4 / \delta)} + \log^{2\theta} \lrb{ {12 e /  \delta} } } \,,
    \end{align*}
    where we have used that \Cref{assum:weibull} implies \Cref{assum:variance} with $\sigma^2 = 2 \Gamma(2\theta + 1) \wnorm^2$ thanks to \Cref{lem:subw-moments}. This proves the first inequality in the statement. Going further, we can use the fact that $2ab = \inf_{r > 0} a^2/r + r b^2$ for any $a,b > 0$ to get that
    \begin{align*}
        2 \wnorm D \Brb{ \sqrt{ T \log(4 / \delta)} + \log^{\theta} \brb{4 e T/ \delta }} &\leq \frac{D^2}{\eta_T} + \eta_T \phi^2 \Brb{ \sqrt{ T \log(4 / \delta)} + \log^{\theta} \brb{4 e T/ \delta }}^2 \\
        &\leq \frac{D^2 \sqrt{T}}{\eta} + \frac{2 \eta}{\sqrt{T}} \phi^2 \Brb{ T \log(4 / \delta) + \log^{2\theta} \brb{4 e T/ \delta }} \,,
    \end{align*}
    implying that
    \begin{align*}
        T\brb{f(\bar{x}_T) - f^*} &\leq   C_1 \frac{D^2 \sqrt{T}}{\eta} +  C_2 \eta (G^2 + \wnorm^2) \sqrt{T} + C_3 \eta \phi^2 \bbrb{ \sqrt{T} \log(4 / \delta) + \frac{\log^{2\theta} \brb{4 e T/ \delta }}{\sqrt{T}}} 
        \\&\hspace{20em}+ C_4 \eta \phi^2 \Brb{ \sqrt{\log(eT) \log(4 / \delta)} + \log^{2\theta} \lrb{ {12 e /  \delta} } } \\
        &\leq C_1 \frac{D^2 \sqrt{T}}{\eta} +  C_2 \eta (G^2 + \wnorm^2) \sqrt{T} + C_3 \eta \phi^2 \bbrb{ \sqrt{T} \log(4 / \delta) + \frac{\log^{2\theta} \brb{4 e T/ \delta }}{\sqrt{T}} + \log^{2\theta} \lrb{ {12 e /  \delta} }} \,.
    \end{align*}
\end{proof}

\begin{corollary}\label{cor:poly-bounded}
    Assume that there exits $D>0$ such that $\sqrt{\breg{x}{y}} \leq D$ for any $(x,y) \in \dom(\psi) \times \inter(\dom(\psi))$.
    Then, for any $\delta \in (0,1)$ and $\eta > 0$, \Cref{alg:smd} with $\eta_t = \frac{\eta}{\sqrt{t}}$ satisfies, under \Cref{assum:reg,assum:lip,assum:poly}, that with probability at least $1-\delta$,
    \begin{align*}
        f(\Bar{x}_T) - f^* &\leq \frac{C_1}{\sqrt{T}} \bbrb{ \frac{D^2}{\eta} + \eta G^2 + \wnorm D \bbrb{ \sqrt{\log(e/\delta)} + \frac{\lrb{1/ \delta}^{1/p}}{T^{1/2-1/p}}} 
        + \eta \phi^2 \bbrb{ 1 + \sqrt{ \frac{\log(eT) \log(e / \delta)}{T}} + \frac{\lrb{1 / \delta}^{2/p}}{\sqrt{T}} } } \\
        &\leq \frac{C_2}{\sqrt{T}} \bbrb{ \frac{D^2}{\eta} + \eta G^2 
        + \eta \phi^2 \bbrb{ \log(e/\delta) +  \frac{\lrb{1 / \delta}^{2/p}}{\sqrt{T}}} } \,,
    \end{align*}
    where $C_1$ and $C_2$ are constants depending only on $p$.
\end{corollary}
\begin{proof}
    Similar to the proof of \Cref{cor:weibullavg-bounded}, we define $u_t = (x_t - x^*)/(\sqrt{2}D)$ (which satisfies $\|u_t\| \leq 1$), and consider once again the two martingale terms
    \[W_k = \sum_{t=1}^k \ban{ \xi_t, u_t} \qquad \text{and} \qquad V_k = \sum_{t=1}^k \eta_t \brb{\|\xi_t\|_*^2 - \E_t \|\xi_t\|_*^2} \,.\]
    For $(W_k)$, we use \Cref{lem:shortcut-poly-conc}(i) with $\omega_t=1$, while for $(V_k)$, we use the \Cref{lem:shortcut-poly-conc}(ii) with $\omega_t = \eta / \sqrt{t}$ yielding that
    \begin{equation*}
        P\lrb{ \max_{k \in [T]} W_k >  \pnorm \sqrt{2 T \log(1/\delta)} + \brb{2+(p/3)} \pnorm \lrb{T / \delta}^{1/p}} \leq \delta \,,
    \end{equation*}
    and
    \begin{equation*}
        P\lrb{ \max_{k \in [T]} V_k >  2 \eta \pnorm^2 \sqrt{2 \sum_{t=1}^T  (1/t) \log(1/\delta)} + 2 \brb{2+(p/6)} \eta \pnorm^2 \lrb{\sum_{t=1}^T (1/t)^{p/4} / \delta}^{2/p}} \leq \delta \,.
    \end{equation*}
    Since $\sum_{t=1}^T \frac{1}{\sqrt{t}} \leq 2\sqrt{T}$, $\sum_{t=1}^T (1/t) \leq \log(eT)$, and $\sum_{t=1}^T (1/t)^{p/4} \leq p/(p-4)$, \Cref{lem:bounded-avg} implies via a union bound that with probability at least $1-\delta$,
    \begin{align*}
        T\brb{f(\bar{x}_T) - f^*} &\leq   \frac{D^2 \sqrt{T}}{\eta} + 2 \eta (G^2 + \wnorm^2) \sqrt{T} + \sqrt{2} \wnorm D \Brb{ \sqrt{2 T \log(2/\delta)} + \brb{2+(p/3)}  \lrb{2T / \delta}^{1/p}} 
        \\&\hspace{10em}+ 2 \eta \phi^2 \Brb{ \sqrt{2 \log(eT) \log(2 / \delta)} + \brb{2+(p/6)} \brb{p/(p-4)}^{2/p} \lrb{2 / \delta}^{2/p} } \,,
    \end{align*}
    where we have used that \Cref{assum:poly} implies \Cref{assum:variance} with $\sigma^2 = \wnorm^2$. This proves the first inequality in the corollary's statement. For the second, we use once again that $2ab = \inf_{r > 0} a^2/r + r b^2$ for any $a,b > 0$, which implies that
    \begin{align*}
        \sqrt{2} \wnorm D \Brb{ \sqrt{2 T \log(2/\delta)} + \brb{2+(p/3)}  \lrb{2T / \delta}^{1/p}} &\leq \frac{D^2}{\eta_T} + \frac{1}{2}\eta_T \phi^2 \Brb{ \sqrt{2 T \log(2/\delta)} + \brb{2+(p/3)}  \lrb{2T / \delta}^{1/p}}^2 \\
        &\leq \frac{D^2 \sqrt{T}}{\eta} + \frac{\eta}{\sqrt{T}} \phi^2 \Brb{2 T \log(2/\delta) + \brb{2+(p/3)}^2  \lrb{2T / \delta}^{2/p}} \,,
    \end{align*}
    using which we obtain that
    \begin{align*}
        T\brb{f(\bar{x}_T) - f^*} &\leq   \frac{2 D^2 \sqrt{T}}{\eta} + 2 \eta (G^2 + \wnorm^2) \sqrt{T} + \eta \phi^2 \Brb{2 \sqrt{T}\log(2/\delta) + \brb{2+(p/3)}^2 T^{(4-p)/(2p)} \lrb{2/ \delta}^{2/p}}  
        \\&\hspace{10em}+ 2 \eta \phi^2 \Brb{ \sqrt{2 \log(eT) \log(2 / \delta)} + \brb{2+(p/6)} \brb{p/(p-4)}^{2/p} \lrb{2 / \delta}^{2/p} } \\
        &\leq \frac{2 D^2 \sqrt{T}}{\eta} + 2 \eta (G^2 + \wnorm^2) \sqrt{T} + 3 \eta \phi^2 \Brb{2 \sqrt{T}\log(2/\delta) + \brb{2+(p/3)}^2 \brb{p/(p-4)}^{2/p} \lrb{2/ \delta}^{2/p}} \,,
    \end{align*}
    where in the second step we used that $p>4$.
\end{proof}

\section{Proofs of Section \ref{sec:last}} \label{app:last}

\subsection{Proof of Proposition \ref{thm:chicken-egg}}
\chickenegg*
\begin{proof}
    For any $\lambda \in \R$ and $0 \leq t \leq n$, define\footnote{One can set $\lan{M}_0 = [M]_0 = M_0$.}
    \begin{equation*}
        V_t(\lambda) = \exp\bbrb{\lambda M_t - \frac{\lambda^2}{2}\brb{\lan{M}_t + [M]_t}} \,.
    \end{equation*}
    By Lemma B.1 in \citep{Bercu2008}, $(V_t(\lambda))_{t=0}^n$ is a (non-negative) supermartingale (with $V_0(\lambda)=1$). For $t \in [n]$, define the event $A_t = \bcb{M_t \geq x \:\text{and}\: \lan{M}_t + [M]_t \leq \alpha M_t + \beta }$.
    From the proof of Theorem 3.3 in \citep{Harvey2019}, if we fix some $\lambda \in (0, 1/(2\alpha))$, then there exists  $c=c(\lambda, \alpha) \in (0,2]$ such that $\brb{\lambda + c \lambda^2 \alpha}^2 = 2 c \lambda ^2$. 
    With this in mind, we have that for any $t \in [n]$ and any $\lambda \in (0, 1/(2\alpha))$:
    \begin{align*}
        \mathbb{I}\{A_{t}\}
        &\leq  \exp\bbrb{(\lambda + c \lambda^2 \alpha) M_{t} - c \lambda^2\brb{\lan{M}_{t} + [M]_{t}} - \lambda x + c \lambda^2 \beta} \\
        &= \exp\brb{- \lambda x + c \lambda^2 \beta}  \exp\bbrb{(\lambda + c \lambda^2 \alpha) M_{t} - c \lambda^2\brb{\lan{M}_{t} + [M]_{t}}} \\
        &= \exp\brb{- \lambda x + c \lambda^2 \beta} \exp\bbrb{\tilde{\lambda} M_{t} - \frac{\tilde{\lambda}^2}{2}\brb{\lan{M}_{t} + [M]_{t}}} \\
        &= \exp\brb{- \lambda x + c \lambda^2 \beta} V_{t}(\tilde{\lambda}) \leq  \exp\brb{- \lambda x + 2 \lambda^2 \beta} V_{t}(\tilde{\lambda})\,, 
    \end{align*}
    where $\tilde{\lambda} = \lambda + c \lambda^2 \alpha$, and the first inequality holds since the argument of the exponent is non-negative under $A_{t}$.
    Hence, \Cref{lem:maximal} entails that
    \begin{equation*}
        P\lrb{\bigcup_{t=1}^n A_t} \leq \exp\brb{- \lambda x + 2 \lambda^2 \beta} \,.
    \end{equation*}
    Finally, upon choosing $\lambda = \min\bcb{\frac{x}{4\beta},\frac{1}{3\alpha}}$, we can conclude that
    \[ \exp\brb{- \lambda x + 2 \lambda^2 \beta} \leq \exp\bbrb{-\min\bbcb{\frac{x^2}{8\beta}, \frac{x}{6\alpha}}} \,.\]
\end{proof}

\subsection{Proof of Lemma \ref{lem:last-iterate-unrolled}}
\lastiterateunrolled*
\begin{proof}
    For any $k \in [T-1]$, \Cref{lem:smd-iterate-comparison} with $j=T-k$ and $r=T$ implies that
    \begin{equation*}
        \sum_{t=T-k}^T \brb{f(x_t) - f(x_{T-k})} \leq \sum_{t=T-k}^T \langle \xi_t, x_t - x_{T-k} \rangle + \frac{1}{2} \sum_{t=T-k}^T \eta_t \|\hat{g}_t\|^2_* + \sum_{t=T-k}^T \tilde{\eta}_t \breg{x_{T-k}}{x_t} \,,
    \end{equation*}
    where $\tilde{\eta}_t = 1 / \eta_t -  1/ \eta_{t-1}$ and $\tilde{\eta}_1 = 1/\eta_1$. We then proceed as in the proof of Lemma 7.1 in \citep{Harvey2019}. Namely, we define $S_k = \frac{1}{k+1}\sum_{t=T-k}^T f(x_t)$, which, combined with the previous inequality, yields that
    \begin{align*}
        S_{k-1} &= S_k + \frac{S_k - f(x_{T-k})}{k} \\
        &\leq S_k + \frac{1}{k(k+1)} \sum_{t=T-k}^T \langle \xi_t, x_t - x_{T-k} \rangle + \frac{1}{2k(k+1)} \sum_{t=T-k}^T \eta_t \|\hat{g}_t\|^2_* + \frac{1}{k(k+1)} \sum_{t=T-k}^T \tilde{\eta}_t \breg{x_{T-k}}{x_t} \,.
    \end{align*}
    Since $S_0 = f(x_T)$, by unrolling the recursion we obtain that
    \begin{multline} \label{last-iterate-unrolled}
        f(x_T) \leq \frac{1}{\floor{T/2}+1} \sum_{t=\ceil{T/2}}^T f(x_t) + \sum_{k=1}^{\floor{T/2}}\frac{1}{k(k+1)} \sum_{t=T-k}^T \langle \xi_t, x_t - x_{T-k} \rangle + \sum_{k=1}^{\floor{T/2}} \frac{1}{2k(k+1)} \sum_{t=T-k}^T \eta_t \|\hat{g}_t\|^2_* \\+ \sum_{k=1}^{\floor{T/2}}  \frac{1}{k(k+1)} \sum_{t=T-k}^T \tilde{\eta}_t \breg{x_{T-k}}{x_t} \,.
    \end{multline}
    One can rewrite the second term on the right-hand side of the above inequality as follows
    \begin{align*}
        \sum_{k=1}^{\floor{T/2}}\frac{1}{k(k+1)} \sum_{t=T-k}^T \langle \xi_t, x_t - x_{T-k} \rangle &= \sum_{t=\ceil{T/2}}^T \sum_{k=(T-t)\lor 1}^{\floor{T/2}} \frac{1}{k(k+1)} \ban{\xi_t,  x_t- x_{T-k}} \\
        &= \sum_{t=\ceil{T/2}}^T \sum_{j=\ceil{T/2}}^{t\land (T-1)} \frac{1}{(T-j)(T-j+1)} \ban{\xi_t,  x_t - x_{j}} = \sum_{t=\ceil{T/2}}^T \ban{\xi_t, w_t}\,.
    \end{align*}
    Similarly, we also have that
    \begin{align*}
        \sum_{k=1}^{\floor{T/2}} \frac{1}{2k(k+1)} \sum_{t=T-k}^T \eta_t \|\hat{g}_t\|^2_* 
        &= \frac{1}{2} \sum_{t=\ceil{T/2}}^T \eta_t \|\hat{g}_t\|^2_* \sum_{j=\ceil{T/2}}^{t\land (T-1)}  \frac{1}{(T-j)(T-j+1)} = \frac{1}{2} \sum_{t=\ceil{T/2}}^T \eta_t \rho_t \|\hat{g}_t\|^2_*  \\
        \sum_{k=1}^{\floor{T/2}}  \frac{1}{k(k+1)} \sum_{t=T-k}^T \tilde{\eta}_t \breg{x_{T-k}}{x_t} &= \sum_{t=\ceil{T/2}}^T \tilde{\eta}_t \sum_{j=\ceil{T/2}}^{t\land (T-1)} \frac{1}{(T-j)(T-j+1)} \breg{x_{j}}{x_t} = \sum_{t=\ceil{T/2}}^T \tilde{\eta}_t z_t \,.
    \end{align*}
    After plugging these expressions back into \eqref{last-iterate-unrolled}, we conclude the proof by using that $\floor{T/2}+1 \geq T/2$ and observing that for any time-step $t \geq \ceil{T/2}$,
    \begin{align*}
        \eta_t  = \frac{\eta}{\sqrt{t}} \leq \frac{\sqrt{2}\eta}{\sqrt{T}} \qquad \text{and} \qquad \tilde{\eta}_t = \frac{1}{\eta} (\sqrt{t}-\sqrt{t-1}) = \frac{1}{\eta(\sqrt{t}+\sqrt{t-1})} \leq \frac{\sqrt{2}}{\eta\sqrt{T}} \,.
    \end{align*}
\end{proof}

\subsection{Proof of Lemma \ref{lem:max-z}}
Recall that for a time-step $s$ such that 
$\ceil{T/2} \leq s \leq T$, $Q_s = \sum_{t=\ceil{T/2}}^s \lan{\xi_t,  w_t}$, and that $z^*$ is short for $\max_{\ceil{T/2} \leq s \leq T} z_s$.
\maxzlemma*
\begin{proof}
    Notice that $z_{\ceil{T/2}}=0$ and $Q_{\ceil{T/2}} = 0$; hence, the lemma trivially holds when $T = 1$. Thus, we assume for what follows that $T \geq 2$.
    Let $j$ and $s$ be two time-steps such that $\ceil{T/2} + 1 \leq s \leq T$ and $\ceil{T/2} \leq j \leq s$. Then, via \Cref{lem:smd-iterate-comparison}, we have that
    \begin{align*}
        \frac{1}{ \eta_{s-1}} \breg{x_j}{x_{s}}  \leq \sum_{t=j}^{s-1} \brb{f(x_j) - f(x_t)} + \sum_{t=j}^{s-1} \langle \xi_t, x_t - x_j \rangle + \frac{1}{2} \sum_{t=j}^{s-1} \eta_t \|\hat{g}_t\|^2_* + \sum_{t=j}^{s-1} \tilde{\eta}_t \breg{x_j}{x_t} \,,
    \end{align*}
    where $\tilde{\eta}_t = 1 / \eta_t -  1/ \eta_{t-1}$ and $\tilde{\eta}_1 = 1/\eta_1$. This, in turn, implies that
    \begin{multline} \label{generic-z-bound-1}
        \frac{1}{\eta_{s-1}} \sum_{j=\ceil{T/2}}^{s\land (T-1)} \alpha_j \breg{x_j}{x_{s}} \leq \sum_{j=\ceil{T/2}}^{s\land (T-1)} \alpha_j \sum_{t=j}^{s-1} \brb{f(x_j) - f(x_t)} + \sum_{j=\ceil{T/2}}^{s\land (T-1)} \alpha_j \sum_{t=j}^{s-1} \langle \xi_t, x_t - x_j \rangle \\+ \frac{1}{2} \sum_{j=\ceil{T/2}}^{s\land (T-1)} \alpha_j  \sum_{t=j}^{s-1} \eta_t \|\hat{g}_t\|^2_* + \sum_{j=\ceil{T/2}}^{s\land (T-1)} \alpha_j  \sum_{t=j}^{s-1} \tilde{\eta}_t \breg{x_j}{x_t} \,.
    \end{multline}
    For the last three terms, we swap the sums obtaining that
    \begin{align*}
        \sum_{j=\ceil{T/2}}^{s\land (T-1)} \alpha_j \sum_{t=j}^{s-1} \langle \xi_t, x_t - x_j \rangle &= \sum_{t=\ceil{T/2}}^{s-1} \sum_{j=\ceil{T/2}}^{t \land (T-1)} \alpha_j \langle \xi_t, x_t - x_j \rangle = \sum_{t=\ceil{T/2}}^{s-1} \langle \xi_t, w_t \rangle \\
        \frac{1}{2} \sum_{j=\ceil{T/2}}^{s\land (T-1)} \alpha_j  \sum_{t=j}^{s-1} \eta_t \|\hat{g}_t\|^2_* &= \frac{1}{2} \sum_{t=\ceil{T/2}}^{s-1} \eta_t \|\hat{g}_t\|^2_* \sum_{j=\ceil{T/2}}^{t \land (T-1)} \alpha_j = \frac{1}{2} \sum_{t=\ceil{T/2}}^{s-1} \eta_t \rho_t \|\hat{g}_t\|^2_* \leq \frac{\eta}{\sqrt{2T}} \sum_{t=\ceil{T/2}}^{s-1} \rho_t \|\hat{g}_t\|^2_* \\
        \sum_{j=\ceil{T/2}}^{s\land (T-1)} \alpha_j  \sum_{t=j}^{s-1} \tilde{\eta}_t \breg{x_j}{x_t} &=  \sum_{t=\ceil{T/2}}^{s-1} \tilde{\eta}_t \sum_{j=\ceil{T/2}}^{t \land (T-1)} \alpha_j  \breg{x_j}{x_t} =  \sum_{t=\ceil{T/2}}^{s-1} \tilde{\eta}_t z_t \,.
    \end{align*}
    For the first term, if we define $\Delta_t = f(x_t) - f^*$ for $t \in [T]$, we obtain that
    \begin{align*}
        \sum_{j=\ceil{T/2}}^{s\land (T-1)} \alpha_j \sum_{t=j}^{s-1} \brb{f(x_j) - f(x_t)} &= \sum_{j=\ceil{T/2}}^{s\land (T-1)} \alpha_j \sum_{t=j}^{s-1} \brb{\Delta_j - \Delta_t} \\
        &= \sum_{j=\ceil{T/2}}^{s-1} \alpha_j \Delta_j (s-j)  - \sum_{j=\ceil{T/2}}^{s-1} \alpha_j \sum_{t=j}^{s-1} \Delta_t \\
        &= \sum_{t=\ceil{T/2}}^{s-1} \alpha_t \Delta_t (s-t)  - 
        \sum_{t=\ceil{T/2}}^{s-1} \Delta_t \sum_{j=\ceil{T/2}}^{t} \alpha_j \\
        &= \sum_{t=\ceil{T/2}}^{s-1} \Delta_t \bbrb{ \frac{s-t}{(T-t)(T-t+1)} - \frac{1}{T-t} + \frac{1}{T-\ceil{T/2}+1}} \\
        &\leq \frac{1}{\floor{T/2}+1} \sum_{t=\ceil{T/2}}^{s-1} \Delta_t \leq \frac{2}{T} \sum_{t=\ceil{T/2}}^{s-1} \Delta_t \,,
    \end{align*}
    where in the second equality we used that the inner sum is empty when $j=s$ and that $s \leq T$, the fourth equality follows from \Cref{lem:alpha-sum} and the definition of $\alpha_t$,
    and the inequality holds since $(s-t)/(T-t+1) < 1$. Returning back to \Cref{generic-z-bound-1}, we have that
    \begin{align*}
        z_s &= \sum_{j=\ceil{T/2}}^{s\land (T-1)} \alpha_j \breg{x_j}{x_{s}} \leq \frac{\eta_{\ceil{T/2}}}{\eta_{s-1}} \sum_{j=\ceil{T/2}}^{s\land (T-1)} \alpha_j \breg{x_j}{x_{s}} \\
        &\leq \eta_{\ceil{T/2}} \biggl( \frac{2}{T} \sum_{t=\ceil{T/2}}^{s-1} \brb{f(x_t) - f^*} + \sum_{t=\ceil{T/2}}^{s-1} \langle \xi_t, w_t \rangle + \frac{\eta}{\sqrt{2T}} \sum_{t=\ceil{T/2}}^{s-1} \rho_t \|\hat{g}_t\|^2_* + \sum_{t=\ceil{T/2}}^{s-1} \tilde{\eta}_t z_t \biggr) \\
        &\leq \frac{2\sqrt{2} \eta}{T\sqrt{T}} \sum_{t=\ceil{T/2}}^{s-1} \brb{f(x_t) - f^*} + \frac{\sqrt{2} \eta}{\sqrt{T}} \sum_{t=\ceil{T/2}}^{s-1} \langle \xi_t, w_t \rangle + \frac{\eta^2}{T} \sum_{t=\ceil{T/2}}^{s-1} \rho_t \|\hat{g}_t\|^2_* + \frac{\sqrt{2} \eta}{\sqrt{T}} \sum_{t=\ceil{T/2}}^{s-1} \tilde{\eta}_t z_t \,.
    \end{align*}

    Notice that the terms in the first, third and fourth sum on the right-hand side of the last inequality are non-negative. Hence, it holds that
    \begin{multline*} 
        z^* \leq  \frac{2\sqrt{2} \eta}{T\sqrt{T}} \sum_{t=\ceil{T/2}}^{T-1} \brb{f(x_t) - f^*} + \frac{\sqrt{2} \eta}{\sqrt{T}} \max_{\ceil{T/2} \leq n \leq T-1} \sum_{t=\ceil{T/2}}^{n} \langle \xi_t, w_t \rangle + \frac{\eta^2}{T} \sum_{t=\ceil{T/2}}^{T-1} \rho_t \|\hat{g}_t\|^2_* 
         + \frac{\sqrt{2} \eta}{\sqrt{T}} \sum_{t=\ceil{T/2}}^{T-1} \tilde{\eta}_t z_t \,.
    \end{multline*}
    Next, we will bound the last term by relating it back to $z^*$. Since this term is zero when $T=2$ (recalling that $z_{\ceil{T/2}}=0$), we focus in the following argument on the case when $T \geq 3$. Observe that
    \begin{align*}
        \frac{\sqrt{2} \eta}{\sqrt{T}} \sum_{t=\ceil{T/2}}^{T-1} \tilde{\eta}_t  
        &=   \frac{\sqrt{2} \eta}{\sqrt{T}} \sum_{t=\ceil{T/2}}^{T-1} \frac{\sqrt{t} - \sqrt{t-1}}{\eta} \\
        &=   \frac{\sqrt{2}}{\sqrt{T}} \brb{\sqrt{T-1} - \sqrt{\ceil{T/2}-1}} \\
        &\leq  \frac{\sqrt{2}}{\sqrt{T}} \brb{\sqrt{T-1} - \sqrt{T/2-1}} \\
    \end{align*}
    As a function of $T$, the last expression is decreasing in $T \geq 3$, and thus (by plugging in $T=3$) can be bounded by $1/\sqrt{3}$ . Hence,
    \begin{equation*}
        \frac{\sqrt{2} \eta}{\sqrt{T}} \sum_{t=\ceil{T/2}}^{T-1} \tilde{\eta}_t z_t \leq \frac{1}{\sqrt{3}} z^* \leq \frac{2}{3} z^* \,.
    \end{equation*}
    Consequently,
    \begin{align*}
        z^* &\leq  \frac{6\sqrt{2} \eta}{T\sqrt{T}} \sum_{t=\ceil{T/2}}^{T-1} \brb{f(x_t) - f^*} + \frac{3 \sqrt{2} \eta}{\sqrt{T}} \max_{\ceil{T/2} \leq n \leq T-1} \sum_{t=\ceil{T/2}}^{n} \langle \xi_t, w_t \rangle + \frac{3\eta^2}{T} \sum_{t=\ceil{T/2}}^{T-1} \rho_t \|\hat{g}_t\|^2_* \\
        &\leq \frac{6\sqrt{2} \eta}{T\sqrt{T}} \sum_{t=\ceil{T/2}}^{T} \brb{f(x_t) - f^*} + \frac{3 \sqrt{2} \eta}{\sqrt{T}} Q_{n^*} + \frac{3\eta^2}{T} \sum_{t=\ceil{T/2}}^{T} \rho_t \|\hat{g}_t\|^2_* \,.
    \end{align*}
\end{proof}

\subsection{Proof of Lemma \ref{lem:tqv-tcv-bound}}
\tqvtcv*
\begin{proof}
    Recall that $w_t = \sum_{j=\ceil{T/2}}^{t\land (T-1)} \alpha_j (x_t - x_{j})$, $z_t = \sum_{j=\ceil{T/2}}^{t\land (T-1)} \alpha_j \breg{x_{j}}{x_t}$, and $\rho_t = \sum_{j=\ceil{T/2}}^{t\land (T-1)} \alpha_j$ for time-step $t \geq \ceil{T/2}$, and observe that 
    \begin{align} \label{w-z-relation}
        \|w_t\|^2  
        = \rho_t^2 \lno{\sum_{j=\ceil{T/2}}^{t\land (T-1)} \frac{\alpha_j}{\rho_t} (x_t - x_{j})}^2 
        \leq \rho_t^2 \sum_{j=\ceil{T/2}}^{t\land (T-1)} \frac{\alpha_j}{\rho_t} \|x_t - x_{j}\|^2
        \leq 2 \rho_t \sum_{j=\ceil{T/2}}^{t\land (T-1)} \alpha_j \breg{x_{j}}{x_t} 
        = 2 \rho_t z_t \,,
    \end{align}
    where the first inequality holds via the convexity of $\|\cdot\|^2$, and the second follows from the fact that $\|x_t-x_j\|^2 \leq 2 \breg{x_{j}}{x_t}$ as $\psi$ is $1$-strongly convex.
    Hence,
    \begin{align*}
        \lan{Q}_{T} + [Q]_{T} &= \sum_{t=\ceil{T/2}}^{T} \Brb{ \E_t\bsb{\labs{\lan{\xi_t, w_t}}^2} + \labs{\lan{\xi_t, w_t}}^2} \\
        &\leq \sum_{t=\ceil{T/2}}^{T} \Brb{ \E_{t}\bsb{\|\xi_t\|^2_* \|w_t\|^2} + \|\xi_t\|^2_* \|w_t\|^2} \\
        &= \sum_{t=\ceil{T/2}}^{T} \|w_t\|^2 \Brb{ \E_{t}\bsb{\|\xi_t\|^2_*} + \|\xi_t\|^2_*} \\
        &\leq 2 z^* \sum_{t=\ceil{T/2}}^{T} \rho_t \Brb{ \E_{t}\bsb{\|\xi_t\|^2_*} + \|\xi_t\|^2_*} \\
        &= 4 z^* \sum_{t=\ceil{T/2}}^{T} \rho_t  \E_{t}\bsb{\|\xi_t\|^2_*} + 2 z^* \sum_{t=\ceil{T/2}}^{T} \rho_t \Brb{\|\xi_t\|^2_* - \E_{t}\bsb{\|\xi_t\|^2_*}} \\
        &\leq 4 \sigma^2 z^* \sum_{t=\ceil{T/2}}^{T} \rho_t + 2 z^* \sum_{t=\ceil{T/2}}^{T} \rho_t \Brb{\|\xi_t\|^2_* - \E_{t}\bsb{\|\xi_t\|^2_*}} \\
        &\leq 4 \sigma^2 z^* \log(4T) + 2 z^* \sum_{t=\ceil{T/2}}^{T} \rho_t \Brb{\|\xi_t\|^2_* - \E_{t}\bsb{\|\xi_t\|^2_*}} \,,
    \end{align*}
    where the first inequality follows from the definition of the dual norm, the second equality holds since $w_t$ is $\F_{t-1}$-measurable, the second inequality follows from \eqref{w-z-relation} and the definition of $z^*$, the third inequality follows from \Cref{assum:variance}, and the last inequality is an application of \Cref{lem:alpha-double-sum}. 
\end{proof}

\subsection{Proof of Theorem \ref{thm:last}}
\thmlast*
\begin{proof}
From \Cref{lem:last-iterate-unrolled}, we have that
\begin{align*}
    f(x_T) - f^* \leq \frac{2}{T} \sum_{t=\ceil{T/2}}^T \brb{f(x_t) - f^*} + Q_T + \frac{\eta}{\sqrt{2T}} \sum_{t=\ceil{T/2}}^T \rho_t \|\hat{g}_t\|^2_* + \frac{\sqrt{2}}{\eta \sqrt{T}} \sum_{t=\ceil{T/2}}^T z_t \,.
\end{align*}
Notice that
\begin{align*}
    \sum_{t=\ceil{T/2}}^T \rho_t \|\hat{g}_t\|_*^2 
    &= \sum_{t=\ceil{T/2}}^T \rho_t \|g_t - \xi_t\|_*^2  \\
    &\leq 2 \sum_{t=\ceil{T/2}}^T \rho_t \brb{\|g_t\|_*^2 + \|\xi_t\|_*^2} \\
    &= 2 \sum_{t=\ceil{T/2}}^T \rho_t \brb{\|g_t\|_*^2 + \E_t \|\xi_t\|_*^2} + 2 \sum_{t=\ceil{T/2}}^T \rho_t \brb{\|\xi_t\|_*^2 - \E_t \|\xi_t\|_*^2} \\
    &\leq 2 \sum_{t=\ceil{T/2}}^T \rho_t \brb{G^2 + \sigma^2} + 2 \sum_{t=\ceil{T/2}}^T \rho_t \brb{\|\xi_t\|_*^2 - \E_t \|\xi_t\|_*^2} \\
    &\leq 2 \brb{G^2 + \sigma^2} \log(4T) + 2 \sum_{t=\ceil{T/2}}^T \rho_t \brb{\|\xi_t\|_*^2 - \E_t \|\xi_t\|_*^2} \,,
\end{align*}
where the second inequality follows from \Cref{assum:lip,assum:variance}, and the third inequality follows from \Cref{lem:alpha-double-sum}.
For what follows, define
\begin{align*}
    \term_1 &= \frac{1}{\sqrt{T}} \sum_{t=1}^T \brb{f(x_t) - f^*}  \\
    \term_2 &= \sum_{t=\ceil{T/2}}^T \rho_t \brb{\|\xi_t\|_*^2 - \E_t \|\xi_t\|_*^2} \,.
\end{align*}
From the assumption in the theorem's statement, we have that for any $\delta \in (0,1)$,
\begin{equation} \label{last-iterate-prob-bound-1and2}
    P\brb{\term_1 > \Xi_1(\delta)} \leq \delta \qquad \text{and} \qquad P\brb{\term_2 > \Xi_2(\delta)} \leq \delta \,.
\end{equation}
Additionally, define $\exterm_3 = \brb{G^2 + \sigma^2} \log(4T)$.
Subsequently, it holds that
\begin{align} \label{simple-last-iterate-with-z}
    f(x_T) - f^* \leq \frac{2}{\sqrt{T}} \term_1 + Q_{n^*} + \frac{\sqrt{2}\eta}{\sqrt{T}} (\term_2 + \exterm_3) + \frac{\sqrt{2}}{\eta \sqrt{T}} \sum_{t=\ceil{T/2}}^T z_t \,.
\end{align}
On the other hand, we have via \Cref{lem:max-z} that 
\begin{align}
    z^* = \max_{\ceil{T/2} \leq s \leq T} z_s &\leq \frac{6\sqrt{2} \eta}{T\sqrt{T}} \sum_{t=\ceil{T/2}}^{T} \brb{f(x_t) - f^*} + \frac{3\sqrt{2} \eta}{\sqrt{T}} Q_{n^*} + \frac{3\eta^2}{T} \sum_{t=\ceil{T/2}}^{T} \rho_t \|\hat{g}_t\|^2_* \nonumber\\ \label{simple-max-z-bound}
    &\leq \frac{6\sqrt{2} \eta}{T} \term_1 + \frac{3\sqrt{2} \eta}{\sqrt{T}} Q_{n^*} + \frac{6\eta^2}{T} (\term_2 + \exterm_3) \,.
\end{align}
Hence,
\begin{align*}
    \frac{\sqrt{2}}{\eta \sqrt{T}} \sum_{t=\ceil{T/2}}^T z_t &\leq \frac{\sqrt{2} T}{\eta \sqrt{T}}\bbrb{\frac{6\sqrt{2} \eta}{T} \term_1 + \frac{3\sqrt{2} \eta}{\sqrt{T}} Q_{n^*} + \frac{6\eta^2}{T} (\term_2 + \exterm_3)} \\
    &= \frac{12}{\sqrt{T}} \term_1 + 6 Q_{n^*} + \frac{6 \sqrt{2} \eta}{\sqrt{T}} (\term_2 + \exterm_3) \,.
\end{align*}
Plugging back into \eqref{simple-last-iterate-with-z} yields that
\begin{equation} \label{simple-last-iterate}
    f(x_T) - f^* \leq \frac{14}{\sqrt{T}} \term_1 + 7 Q_{n^*} + \frac{7 \sqrt{2} \eta}{\sqrt{T}} (\term_2 + \exterm_3) \,.
\end{equation}
Our aim in the sequel is to use the above inequality in conjunction with \eqref{last-iterate-prob-bound-1and2} and \Cref{thm:chicken-egg} to bound the error in high probability. Towards that end, we start with the following upper bound on the TCV and TQV of $Q_{n^*}$, which is implied by \Cref{lem:tqv-tcv-bound} and the fact that the TCV and TQV are non-decreasing.
\begin{align*}
    \lan{Q}_{n^*} + [Q]_{n^*} \leq \lan{Q}_{T} + [Q]_{T} \leq 4 \sigma^2 z^* \log(4T) + 2 z^* \sum_{t=\ceil{T/2}}^{T} \rho_t \brb{\|\xi_t\|^2_* - \E_{t}\|\xi_t\|^2_*} = 2 z^* (2 \tilde{\exterm}_3 + \term_2) \,,
\end{align*}
where $\tilde{\exterm}_3 \coloneqq \sigma^2 \log(4T)$.
Moreover, under the event that $\term_1 \leq \Xi_1(\delta)$ and $\term_2 \leq \Xi_2(\delta)$,
we have that
\begin{equation} \label{last-iterate-trivial-prob-bound-1} 
    \frac{12\sqrt{2}\eta}{T} \term_1 + \frac{6\sqrt{2} \eta}{\sqrt{T}} Q_{n^*} + \frac{12\eta^2}{T}(\term_2 + \exterm_3)
    \leq \frac{12\sqrt{2}\eta}{T} \Xi_1(\delta) + \frac{6\sqrt{2} \eta}{\sqrt{T}} Q_{n^*} + \frac{12\eta^2}{T}(\Xi_2(\delta) + \exterm_3) 
\end{equation}
and
\begin{equation} \label{last-iterate-trivial-prob-bound-2}
    \term_2 + 2 \tilde{\exterm}_3
    \leq \Xi_2(\delta) + 2 \tilde{\exterm}_3 \,,
\end{equation}
which implies that under the same event,
\begin{align}
    \lan{Q}_{n^*} + [Q]_{n^*} &\leq 2 z^* \brb{\term_2 + 2 \tilde{\exterm}_3} \nonumber \\
    &\leq \bbrb{\frac{12\sqrt{2}\eta}{T} \term_1 + \frac{6\sqrt{2} \eta}{\sqrt{T}} Q_{n^*} + \frac{12\eta^2}{T}(\term_2 + \exterm_3)} \brb{\term_2 + 2 \tilde{\exterm}_3} \nonumber \\
    &\leq \bbrb{\frac{12\sqrt{2}\eta}{T} \term_1 + \frac{6\sqrt{2} \eta}{\sqrt{T}} Q_{n^*} + \frac{12\eta^2}{T}(\term_2 + \exterm_3)} \brb{\Xi_2(\delta) + 2 \tilde{\exterm}_3} \nonumber \\ \label{quad-var-bound-hp}
    &\leq \bbrb{\frac{12\sqrt{2}\eta}{T} \Xi_1(\delta) + \frac{6\sqrt{2} \eta}{\sqrt{T}} Q_{n^*} + \frac{12\eta^2}{T}(\Xi_2(\delta) + \exterm_3)} \brb{\Xi_2(\delta) + 2 \tilde{\exterm}_3} \,,
\end{align}
where the second inequality follows from \eqref{simple-max-z-bound} and the fact that $\term_2 + 2 \tilde{\exterm}_3$ is non-negative,\footnote{As $\tilde{\exterm}_3$ is an upper bound for $\sum_{t=\ceil{T/2}}^{T} \rho_t  \E_{t}\bsb{\|\xi_t\|^2_*}$.}
the third inequality follows from \eqref{last-iterate-trivial-prob-bound-2} and the fact that the first bracketed expression on the left-hand side is non-negative as it is an upper bound for the non-negative quantity 2$z^*$, whereas the last inequality follows from \eqref{last-iterate-trivial-prob-bound-1} and the fact that $\Xi_2(\delta) + 2 \tilde{\exterm}_3$ is non-negative. 
As a last bit of notation, we define
\begin{align*}
    R_1(\delta) &= \frac{6\sqrt{2} \eta}{\sqrt{T}} \brb{\Xi_2(\delta) + 2 \tilde{\exterm}_3}  \\
    R_2(\delta) &= \bbrb{\frac{12\sqrt{2}\eta}{T} \Xi_1(\delta) + \frac{12\eta^2}{T}(\Xi_2(\delta) + \exterm_3)} \brb{\Xi_2(\delta) + 2 \tilde{\exterm}_3} \\
    \zeta(\delta) &= \frac{14}{\sqrt{T}} \Xi_1\brb{\delta} + \frac{7 \sqrt{2} \eta}{\sqrt{T}} \brb{\Xi_2\brb{\delta } + \exterm_3} + 7\sqrt{8 R_2\brb{\delta} \log (\delta)}  + 42 R_1\brb{\delta} \log (\delta) \,,
\end{align*}
and (for any time-step $s$ such that $\ceil{T/2} \leq s \leq T$) the events 
\begin{align*}
    A_1 &=\bcb{\term_1 \leq \Xi_1(\delta/3)} \cap \bcb{\term_2 \leq \Xi_2(\delta/3)} \\
    A_2(s) &= \Bcb{Q_{s} > \sqrt{8 R_2\brb{\delta/3} \log (3/\delta)}  + 6 R_1\brb{\delta/3} \log (3/\delta)}\\
    A_3(s) &= \Bcb{\lan{Q}_{s} + [Q]_{s} \leq R_1(\delta / 3) Q_{s} + R_2(\delta / 3)}\,.
\end{align*}
Now, notice that
\begin{align*}
    P \Brb{f(x_T) - f^* > \zeta(\delta/3)} &= P\Brb{\bcb{f(x_T) - f^* > \zeta(\delta/3)} \cap A_1} + P\Brb{\bcb{f(x_T) - f^* > \zeta(\delta/3)} \cap \overline{A_1}} \\
    &\leq P\Brb{\bcb{f(x_T) - f^* > \zeta(\delta/3)} \cap A_1} + P\brb{\overline{A_1}} \\
    &\stackrel{(a)}{\leq} P\Brb{\bcb{f(x_T) - f^* > \zeta(\delta/3)} \cap A_1} + 2\delta / 3 \\
    &\stackrel{(b)}{\leq} P\lrb{\bbcb{\frac{14}{\sqrt{T}} \term_1 + 7 Q_{n^*} + \frac{7 \sqrt{2} \eta}{\sqrt{T}} (\term_2 + \exterm_3) > \zeta(\delta/3)} \cap A_1} + 2\delta / 3 \\
    &\leq P\lrb{\bbcb{\frac{14}{\sqrt{T}} \Xi_1(\delta/3) + 7 Q_{n^*} + \frac{7 \sqrt{2} \eta}{\sqrt{T}} (\Xi_2(\delta/3) + \exterm_3) > \zeta(\delta/3)} \cap A_1} + 2\delta / 3 \\
    &= P\lrb{A_2(n^*) \cap A_1} + 2\delta / 3 \\
    &\stackrel{(c)}{\leq} P\Brb{A_2(n^*) \cap A_3(n^*)} + 2\delta / 3 \\
    &\leq P\lrb{\bigcup_{s=\ceil{T/2}}^T \brb{A_1(s) \cap A_2(s)}} + 2\delta / 3 \\
    &\stackrel{(d)}{\leq} \delta/3 + 2\delta / 3 = \delta  \,,
\end{align*}
where $(a)$ follows from \eqref{last-iterate-prob-bound-1and2} and a union bound, $(b)$ follows from \eqref{simple-last-iterate}, $(c)$ follows from \eqref{quad-var-bound-hp} and the definitions of $R_1$, $R_2$, and $A_3$, whereas $(d)$ follows from \Cref{thm:chicken-egg} and the fact that $(Q_t)_{t=\ceil{T/2}}^T$ is a (square integrable) martingale adapted to $(\F_t)_{t=\ceil{T/2}}^T$ (with $Q_{\ceil{T/2}} = 0$). 

Hence, with probability at least $1 - \delta$,
\begin{align*}
    \frac{1}{7}\brb{f(x_T) - f^*} &\leq \frac{2}{\sqrt{T}} \Xi_1\brb{\delta/3} + \frac{\sqrt{2} \eta}{\sqrt{T}} \brb{\Xi_2\brb{\delta/3} + \exterm_3} +   \sqrt{8 R_2\brb{\delta/3 } \log (3/\delta)}  + 6 R_1\brb{\delta/3} \log (3/\delta) \\
    &\stackrel{(a)}{=} \frac{2}{\sqrt{T}} \Xi_1\brb{\delta/3} + \frac{\sqrt{2} \eta}{\sqrt{T}} \brb{\Xi_2\brb{\delta/3} + \exterm_3} + \frac{36 \sqrt{2} \eta}{\sqrt{T}} \brb{\Xi_2(\delta/3) + 2 \tilde{\exterm}_3} \log (3/\delta) \\
    &\hspace{2em} + 2\sqrt{2} \sqrt{\frac{2}{\sqrt{T}} \Xi_1(\delta/3) + \frac{\sqrt{2}\eta}{\sqrt{T}}(\Xi_2(\delta/3) + \exterm_3)} \sqrt{\frac{6 \sqrt{2} \eta}{\sqrt{T}} \brb{\Xi_2(\delta/3) + 2 \tilde{\exterm}_3} \log (3/\delta)} \\
    &\stackrel{(b)}{\leq} \frac{2}{\sqrt{T}} \Xi_1\brb{\delta/3} + \frac{\sqrt{2} \eta}{\sqrt{T}} \brb{\Xi_2\brb{\delta/3} + \exterm_3} + \frac{36 \sqrt{2} \eta}{\sqrt{T}} \brb{\Xi_2(\delta/3) + 2 \tilde{\exterm}_3} \log (3/\delta) \\
    &\hspace{2em} + 4 \bbrb{\frac{2}{\sqrt{T}} \Xi_1(\delta/3) + \frac{\sqrt{2}\eta}{\sqrt{T}}(\Xi_2(\delta/3) + \exterm_3)} + \frac{3 \sqrt{2} \eta}{\sqrt{T}} \brb{\Xi_2(\delta/3) + 2 \tilde{\exterm}_3} \log (3/\delta) \\
    &= 5 \bbrb{\frac{2}{\sqrt{T}} \Xi_1(\delta/3) + \frac{\sqrt{2}\eta}{\sqrt{T}}(\Xi_2(\delta/3) + \exterm_3)} + \frac{39 \sqrt{2} \eta}{\sqrt{T}} \brb{\Xi_2(\delta/3) + 2 \tilde{\exterm}_3} \log (3/\delta) \\
    &\stackrel{(c)}{=} 5 \bbrb{\frac{2}{\sqrt{T}} \Xi_1(\delta/3) + \frac{\sqrt{2}\eta}{\sqrt{T}}(\Xi_2(\delta/3) + \brb{G^2 + \sigma^2} \log(4T))} \\
    &\hspace{10em}+ \frac{39 \sqrt{2} \eta}{\sqrt{T}} \Brb{\Xi_2(\delta/3) + 2 \sigma^2 \log(4T)} \log (3/\delta) \\
    &\stackrel{(d)}{\leq} 5 \bbrb{\frac{2}{\sqrt{T}} \Xi_1(\delta/3) + \frac{\sqrt{2}\eta}{\sqrt{T}}G^2 \log(4T)} + \frac{44 \sqrt{2} \eta}{\sqrt{T}} \Brb{\Xi_2(\delta/3) + 2 \sigma^2 \log(4T)} \log (3/\delta)\,,
\end{align*}
where $(a)$ follows from the definitions of $R_1$ and $R_2$, $(b)$ follows from the elementary fact that $a b \leq a^2/2 + b^2/2$, $(c)$ follows from the definitions of $\exterm_3$ and $\tilde{\exterm}_3$, and $(d)$ follows from the fact that $\log(3/\delta) \geq 1$.
We can then conclude that with probability at least $1 - \delta$,
\begin{align*}
    f(x_T) - f^* \leq \frac{35}{\sqrt{T}} \bigg( 2 \Xi_1(\delta/3) + \sqrt{2}\eta G^2 \log(4T) + 9\sqrt{2}\eta \Brb{\Xi_2(\delta/3) + 2 \sigma^2 \log(4T)} \log (3/\delta) \bigg) \,.
\end{align*}
\end{proof}

\subsection{Proof of Corollary \ref{cor:last}}
\corlast*
\begin{proof}
    Starting with case (i), we let $C,C_1,C_2,\dots$ denote positive constants---depending only on $\theta$---whose values may change between steps. In the notation of \Cref{thm:last}, we choose 
    \begin{align*}
        \Xi_1(\delta) = C\log(eT) \bbrb{ \frac{1}{\eta} \breg{x^*}{x_{1}} + \eta \lrb{G^2 + \phi^2 \log^{2 \theta}(e / \delta) }} 
    \end{align*}
    by virtue of \Cref{cor:weibullavg}. While invoking \Cref{lem:shortcut-subw-conc}(ii) with $s=2$ and $\omega_t = \rho_t$ allows us to choose\footnote{This is valid despite the fact that, contrary to \Cref{lem:shortcut-subw-conc}(ii), the indices here start from $\ceil{T/2}$.} 
    \begin{align*}
        \Xi_2(\delta) = C_1 \wnorm^2\sqrt{\sum_{t=\ceil{T/2}}^{T} \rho_t^2  \log(2 / \delta)} + C_2 \wnorm^2 \max_{\ceil{T/2} \leq t \leq T} \rho_t \log^{2\theta} \lrb{ \frac{2 e \sum_{t=\ceil{T/2}}^{T} \rho_t^2}{\max_{\ceil{T/2} \leq t \leq T} \rho_t^2 \delta} } \,.
    \end{align*}
    Then, using that (via \Cref{lem:alpha-double-sum}) 
    \begin{align*}
        \sum_{t=\ceil{T/2}}^{T} \rho_t^2 \leq 3 \qquad \text{and} \qquad \max_{\ceil{T/2} \leq t \leq T} \rho_t = \rho_{T} = \frac{1}{2} \,,
    \end{align*}
    \Cref{thm:last} yields that
    \begin{align*}
        f(x_T) - f^* &\leq \frac{C \log(eT)}{\sqrt{T}} \bigg( \frac{1}{\eta} \breg{x^*}{x_{1}} + \eta \lrb{G^2 + \phi^2 \log^{2 \theta}(e / \delta) } + \eta G^2 
        \\&\hspace{10em}+ \eta \Brb{\wnorm^2 \sqrt{ \log(e/\delta)} + \wnorm^2 \log^{2 \theta}(e/\delta) + \wnorm^2} \log (e/\delta) \bigg) \\
        &\leq \frac{C \log(eT)}{\sqrt{T}} \bigg( \frac{1}{\eta} \breg{x^*}{x_{1}} + \eta \lrb{G^2 + \wnorm^2 \log^{2 \theta+1}(e/\delta) } \bigg) \,,
    \end{align*}
    where upon invoking \Cref{thm:last}, we used the fact that \Cref{assum:weibull} implies \Cref{assum:variance} with $\sigma^2 = 2 \Gamma(2\theta + 1) \wnorm^2$ thanks to \Cref{lem:subw-moments}.

    For case (ii), we let $C$ denote a positive constant---depending only on $p$---whose value may change between steps. Via \Cref{cor:polyavg}, we can choose
    \begin{align*}
        \Xi_1(\delta) = C \log(eT) \bbrb{ \frac{1}{\eta} \breg{x^*}{x_{1}} + \eta \Brb{G^2 + \pnorm^2 \brb{1 / \delta}^{2/p}} } 
    \end{align*}
    Invoking \Cref{lem:shortcut-poly-conc}(ii) with  $\omega_t = \rho_t$ allows us to choose
    \begin{align*}
        \Xi_2(\delta) = 2 \pnorm^2 \sqrt{6 \log(1/\delta)} + 2 \brb{2+(p/6)} \pnorm^2 \brb{3 / \delta}^{2/p} \,,
    \end{align*}
    where we have used that
    \begin{align*}
        \sum_{t=\ceil{T/2}}^{T} \rho_t^{p/2} \leq \sum_{t=\ceil{T/2}}^{T} \rho_t^{2} \leq 3\,,
    \end{align*}
    which holds via \Cref{lem:alpha-double-sum} and the fact that $p>4$ and $\rho_t \leq 1$. \Cref{thm:last} then implies that
    \begin{align*}
        f(x_T) - f^* &\leq \frac{C \log(eT)}{\sqrt{T}} \bigg( \frac{1}{\eta} \breg{x^*}{x_{1}} + \eta \Brb{G^2 + \pnorm^2 \brb{1 / \delta}^{2/p}}  + \eta G^2 
        \\&\hspace{10em}+ \eta \Brb{\pnorm^2 \sqrt{ \log(e/\delta)} + \pnorm^2 \brb{1 / \delta}^{2/p} + \phi^2} \log (e/\delta) \bigg) \\
        &\leq \frac{C \log(eT)}{\sqrt{T}} \bigg( \frac{1}{\eta} \breg{x^*}{x_{1}} + \eta \Brb{G^2 + \pnorm^2 \brb{1 / \delta}^{2/p} \log (e/\delta)} \bigg) \,,
    \end{align*}
    where upon invoking \Cref{thm:last}, we used the fact that \Cref{assum:poly} implies \Cref{assum:variance} with $\sigma^2 = \pnorm^2$; while in the second step, we used the fact that $\sqrt{\log(e/\delta)} \leq \sqrt{(p/4) (e/\delta)^{(4/p)}} = \sqrt{(p/4)} (e/\delta)^{(2/p)}$.
\end{proof}

\subsection{Auxiliary Lemmas}
\begin{lemma} \label{lem:alpha-sum}
    Let $a$ and $b$ be two positive integers such that $a \leq b < T$. Then,
    \begin{equation*}
        \sum_{j=a}^b \frac{1}{(T-j)(T-j+1)} = \frac{1}{T-b} - \frac{1}{T-a+1} \,.
    \end{equation*}
\end{lemma}
\begin{proof}
   \begin{align*}
       \sum_{j=a}^b \frac{1}{(T-j)(T-j+1)} = \sum_{j=a}^b \frac{1}{(T-j)} - \frac{1}{(T-j+1)} = \frac{1}{T-b} - \frac{1}{T-a+1} \,.
   \end{align*}     
\end{proof}

\begin{lemma} \label{lem:alpha-double-sum}
    For $j < T$, let $\alpha_j = \frac{1}{(T-j)(T-j+1)}$. Then, for $T \geq 1$, we have that
    \begin{align*}
        \sum_{t=\ceil{T/2}}^{T} \sum_{j=\ceil{T/2}}^{t \land (T-1)} \alpha_j \leq  \log(4T) \qquad \text{and} \qquad \sum_{t=\ceil{T/2}}^{T} \lrb{\sum_{j=\ceil{T/2}}^{t \land (T-1)}\alpha_j }^2 \leq 3 \,.
    \end{align*}
\end{lemma}
\begin{proof}
    By \Cref{lem:alpha-sum},
    \begin{equation*}
        \sum_{j=\ceil{T/2}}^{t \land (T-1)} \alpha_j \leq \frac{1}{T - t \land (T-1) } \,.
    \end{equation*}
    Assuming $T \geq 2$ (as the lemma follows directly otherwise), we have that
    \begin{align*}
        \sum_{t=\ceil{T/2}}^{T} \sum_{j=\ceil{T/2}}^{t \land (T-1)} \alpha_j &\leq \sum_{t=\ceil{T/2}}^{T} \frac{1}{T - t \land (T-1) } \\
        &= 2 + \sum_{t=\ceil{T/2}}^{T-2} \frac{1}{T - t } \\
        &\leq 2 + \int_{\ceil{T/2}}^{T-1} \frac{1}{T - t } dt \\
        &= 2 + \log(\floor{T/2}) \leq \log(4T) \,.
    \end{align*}
    Similarly,
    \begin{align*}
        \sum_{t=\ceil{T/2}}^{T} \lrb{\sum_{j=\ceil{T/2}}^{t \land (T-1)}\alpha_j }^2 
        &\leq 2 +  \int_{\ceil{T/2}}^{T-1} \frac{1}{(T - t )^2} dt  
        \leq 3\,.
    \end{align*}
\end{proof}

\section{Concentration Inequalities for Martingales With Heavy-Tailed Increments} \label{app:concentration}
We collect in this section relevant concentration results for Martingales with heavy-tailed increments. We treat two families of heavy-tailed random variables: a class of sub-Weibull random variables, and a class of random variables with polynomially decaying tails (implied by a bounded moment assumption). 

\subsection{Sub-Weibull Increments}
Before stating the main concentration inequality in \Cref{lem:subw-martingale-conc}, we collect some basic results concerning sub-Weibull random variables.
The following lemma (adapted from \citep{Madden2020}) provides an upper bound for the $p$-th absolute moment of a sub-Weibull random variable.
\begin{lemma}\label{lem:subw-moments}\cite[Lemma 22]{Madden2020} 
     Let $X$ be a sub-Weibull$(\theta, \wnorm)$ random variable. Then, for any $p>0$, it satisfies
     \begin{equation*}
         \E |X|^p \leq 2 \Gamma(\theta p + 1) \wnorm^p \,.
     \end{equation*}
\end{lemma}
The following lemma shows that centering a random variable preserves the sub-Weibull property up to a constant depending on $\theta$. 
\begin{lemma} \label{lem:weibull-centering}
    Let $X$ be a sub-Weibull$\brb{\theta,\wnorm}$ random variable. Then $X - \E X$ is sub-Weibull$(\theta,c_\theta \wnorm)$,
    where $c_\theta = 2^{\max\{\theta,1\}+1}\Gamma(\theta+1) / \ln^\theta(2)$.
\end{lemma}
\begin{proof}
    If $\theta \leq 1$, define
    \begin{equation*}
        \lno{X}_{\psi_{1/\theta}} = \inf \bbcb{t > 0 \,:\, \E \Bsb{ \exp \Brb{\lrb{|X| / t}^{1/\theta}}} \leq 2} \,
    \end{equation*}
    which is an (Orlicz) norm for the space $L_{\psi_{1/\theta}} = \bcb{X \,:\, \lno{X}_{\psi_{1/\theta}} < \infty}$ \citep[Section~2.7.1]{Vershynin2018}. Clearly, $X$ is sub-Weibull$(\theta, \wnorm)$ if and only if $\lno{X}_{\psi_{1/\theta}} \leq \wnorm$. Starting with the triangle inequality, we proceed in the same manner as in the proof of \citep[Lemma~2.6.8]{Vershynin2018} to get that
    \begin{align*}
         \lno{X - \E X}_{\psi_{1/\theta}} \leq \lno{X}_{\psi_{1/\theta}} + \lno{\E X}_{\psi_{1/\theta}} \leq \wnorm + \frac{|\E X|}{\ln^\theta(2)} \leq \wnorm + \frac{\E | X|}{\ln^\theta(2)} \leq \bbrb{\frac{2\Gamma(\theta+1)}{\ln^\theta(2)}+1}\wnorm \,,
    \end{align*}
    where the last inequality is an application of \Cref{lem:subw-moments}. Hence, the lemma follows for the case when $\theta \leq 1$ after using that $2\Gamma(\theta+1) / \ln^\theta(2) \geq 1$.
    On the other hand, when $\theta > 1$, $\lno{\cdot}_{\psi_{1/\theta}}$ is no longer a norm. Instead, we exploit the fact that $x^{1/\theta}$ is a sub-additive function in $x$ for $\theta > 1$ and $x \geq 0$. In particular, we have that
    \begin{align*}
        \E \lsb{ \exp \lrb{\bbrb{\frac{|X - \E X|}  {c_\theta \wnorm}}^{1/\theta}} } &\leq \E \lsb{ \exp \lrb{\bbrb{\frac{|X| + \E |X|}  {c_\theta \wnorm}}^{1/\theta}} } \\
        &\leq \E \lsb{ \exp \lrb{\bbrb{\frac{|X|}  {c_\theta \wnorm}}^{1/\theta} + \bbrb{\frac{\E |X|}  {c_\theta \wnorm}}^{1/\theta}} } \\
        &= \exp \lrb{\bbrb{\frac{\E |X|}  {c_\theta \wnorm}}^{1/\theta}} \E \lsb{ \lrb{\exp \lrb{\brb{|X| / \wnorm}^{1/\theta}} }^{(1/c_\theta)^{1/\theta}} } \\
        &\leq \exp \lrb{\bbrb{\frac{\E |X|}  {c_\theta \wnorm}}^{1/\theta}}  \lrb{\E \lsb{ 
        \exp \lrb{\brb{|X| / \wnorm}^{1/\theta}} }}^{(1/c_\theta)^{1/\theta}} \\
        &\leq \exp \lrb{\bbrb{\frac{\E |X|}  {c_\theta \wnorm}}^{1/\theta}}  2^{(1/c_\theta)^{1/\theta}} \\
        &\leq \exp \lrb{\bbrb{\frac{2\Gamma(\theta+1)}  {c_\theta}}^{1/\theta}}  2^{(1/c_\theta)^{1/\theta}} \\
        &\leq \exp \lrb{2\bbrb{\frac{2\Gamma(\theta+1)}  {c_\theta}}^{1/\theta}} = 2\,,  \\
    \end{align*}    
where the third inequality is an application of Jensen's inequality as the fact that $0 < (1/c_\theta)^{1/\theta} < 1$ implies the concavity of $x^{(1/c_\theta)^{1/\theta}}$ for $x \geq 0$, the fourth inequality uses that $X$ is sub-Weibull$(\theta, \wnorm)$, the fifth inequality follows via \Cref{lem:subw-moments}, and the last inequality uses the fact that $2\Gamma(\theta+1) \geq 1$.
\end{proof}

The following lemma collects upper bounds for the moment-generating function (MGF) of (centered) sub-Weibull random variables, depending on the value of $\theta$. The MGF of a random variable $X$ is a function of $\lambda \in \R$ given by $\E[\exp(\lambda X)]$.
As mentioned before, our focus in this work is on the heavy-tailed regime where $\theta \geq 1$, though we also consider the canonical case of $\theta=1/2$ for comparison. 
In the latter case, we have the standard bound on the MGF of a sub-Gaussian random variable (see, e.g., \cite[Proposition 2.5.2]{Vershynin2018}). When $\theta = 1$, a similar bound (see, e.g., \cite[Proposition 2.7.1]{Vershynin2018}) holds only for a certain range of $\lambda$. When $\theta > 1$, one cannot bound the MGF in general; thus, we settle for a bound on the MGF of a truncated version of the random variable due to \cite{Bakhshizadeh2023}. This last result is reported in \cite[Lemma 31]{Madden2020} for a specific choice of the truncation parameter, which we will slightly modify when applying this lemma.
\begin{lemma} \label{lem:subw-mgf}
    Let $X$ be a sub-Weibull$\brb{\theta,\wnorm}$ random variable with $\E[X]=0$. 
    \begin{enumerate}[(i)]
        \item \cite[Proposition 2.5.2]{Vershynin2018} If $\theta = 1/2$,
        \begin{equation*} 
            \E[\exp(\lambda X)] \leq \exp(4 e \wnorm^2 \lambda^2) \qquad \forall \lambda \in \R \,.
        \end{equation*}
        \item \cite[Proposition 2.7.1]{Vershynin2018} If $\theta = 1$,
        \begin{equation*} 
            \E[\exp(\lambda X)] \leq \exp(2 e^2 \wnorm^2 \lambda^2) \qquad \forall \lambda: |\lambda| \leq \frac{1}{2 e \wnorm} \,.
        \end{equation*}
        \item If $\theta \geq 1$, let $L = \wnorm h$ for some parameter $h > 0$, and define $\tilde{X} = X \I\{X \leq L\}$. Then,
        \begin{equation*}
            \E\bsb{\exp\brb{\lambda \tilde{X}}} \leq \exp\brb{a \wnorm^2 \lambda^2} \qquad \forall \lambda \in \bbsb{0, \frac{1}{2 h^{1-\frac{1}{\theta}} \wnorm}} \,,
        \end{equation*}
        where
        \[
            a = (2^{2 \theta} + 1) \Gamma(2\theta + 1) +  \frac{2^{3 \theta} \Gamma(3\theta + 1)}{6}  h^{\frac{1}{\theta}-1} \,.
        \]
    \end{enumerate}
\end{lemma}
\begin{proof}~\\
    \textit{(iii)} Since $\E[X]=0$, we have that for any $\lambda \in \Bsb{0, \frac{1}{2 h^{1-\frac{1}{\theta}} \wnorm}}$, 
    \begin{align*}
        \log \E\bsb{\exp\brb{\lambda \tilde{X}}} &\leq \frac{\lambda^2}{2} \bbrb{ \E \Bsb{\tilde{X}^2 \I\{\tilde{X} \leq 0\}} + \E \Bsb{\tilde{X}^2 \exp\brb{\lambda \tilde{X}} \I\{\tilde{X} > 0\}} }\\
        &\leq \frac{\lambda^2}{2} \bbrb{ \E \Bsb{X^2 \I\{X \leq 0\}} + 2^{2 \theta + 1} \Gamma(2\theta + 1) \wnorm^2 +  \frac{2^{3 \theta} \Gamma(3\theta + 1)}{3}  L^{\frac{1}{\theta}-1} \wnorm^{3-\frac{1}{\theta}} }\\
        &= \frac{\lambda^2}{2} \bbrb{ \E \Bsb{X^2 \I\{X \leq 0\}} + 2^{2 \theta + 1} \Gamma(2\theta + 1) \wnorm^2 +  \frac{2^{3 \theta} \Gamma(3\theta + 1)}{3}  h^{\frac{1}{\theta}-1} \wnorm^{2} }\\
        &\leq \frac{\lambda^2}{2} \bbrb{ 2 \Gamma(2 \theta + 1) \wnorm^2 + 2^{2 \theta + 1} \Gamma(2\theta + 1) \wnorm^2 +  \frac{2^{3 \theta} \Gamma(3\theta + 1)}{3}  h^{\frac{1}{\theta}-1} \wnorm^{2} }\,,
    \end{align*}
    where the first inequality follows from Lemma 1 in \citep{Bakhshizadeh2023}, the second inequality follows from Corollary 2 in the same paper,\footnote{In the notation of \cite{Bakhshizadeh2023}, we have that $\alpha = \theta$, $c_\alpha = \wnorm^{-1/\theta}$, $\lambda = \beta I(L) / L$ with $I(L) = (L/\wnorm)^{1/\theta}$ and $\beta \in [0,1/2]$. Compared to Corollary 2 in \citep{Bakhshizadeh2023}, the extra factor of $2$ in the last two terms on the right-hand side of the inequality is because in our case (similar to \cite{Madden2020}),
    we start with the assumption that $X$ is sub-Weibull$(\theta,\wnorm)$, which implies the tail bound $P(|X| \geq t) \leq 2 \exp\lrb{ - I(t)}$.} the equality holds by the definition of $L$, and the last inequality is an application of \Cref{lem:subw-moments}.
\end{proof}
The following proposition provides time-uniform concentration inequalities for martingales with conditionally sub-Weibull increments. Case (i) is a standard sub-Gaussian concentration result included for completeness, whereas Case (ii) considers the heavy-tailed regime where $\theta \geq 1$. The latter generalizes a result in \citep[Proposition 11]{Madden2020}, which corresponds to the case when $s=0$. 
In our problem, this generalized form allows us in come cases to avoid an extra poly-logarithmic dependence on the time horizon, at the cost of a constant depending on $\theta$. This is thanks to the (possibly) non-uniform union bound employed when $s > 0$, which can take advantage of the non-uniformity of the sequence $(m_i)$.
\begin{proposition} \label{lem:subw-martingale-conc}
    Assume that $(X_i)_{i=1}^n$ is a martingale difference sequence adapted to filtration $\mathbb{F}=(\F_i)_{i=0}^n$, where $n$ is a positive integer, and let $S_t = \sum_{i=1}^t X_i$ for $t \in [n]$. Furthermore, assume that for each $i \in [n]$, $X_i$ is sub-Weibull$(\theta,\wnormuni_{i})$ conditioned on $\F_{i-1}$; that is,
    \begin{equation*}
        \E \Bsb{ \exp \Brb{\lrb{|X_i| / \wnormuni_{i}}^{1/\theta}} \, \big| \, \F_{i-1}} \leq 2 \,,
    \end{equation*}
    for some constant $\wnormuni_i > 0$, and define $\wnormuni_* = \max_i \wnormuni_i$.
    Then, for any $\delta \in (0,1)$:
    \begin{enumerate}[(i)]
        \item If $\theta = 1/2$,
        \begin{equation*}
            P\lrb{\bigcup_{t=1}^n \bbcb{S_t \geq 4\sqrt{e \textstyle{\sum_{i=1}^n} \wnormuni_i^2 \log(1 / \delta)}} } \leq \delta \,.
        \end{equation*}
        
        \item If $\theta \geq 1$; then for any $s \geq 0$,
        \begin{equation*}
            P\lrb{\bigcup_{t=1}^n \bbcb{S_t \geq   \sqrt{C_1  \textstyle{\sum_{i=1}^n} \wnormuni_i^2 \log(2 / \delta)} + 4 \wnormuni_* \max \bbcb{ \log^{\theta-1} \lrb{ \frac{2 e \sum_{j=1}^n \wnormuni_j^s}{\wnormuni_*^s \delta} } ,  (s \theta - s)^{\theta-1}} \log(2/\delta) } } \leq \delta  \,,
        \end{equation*}
        where $C_1 = 2^{3 \theta + 1} \Gamma(3\theta + 1)$.
    \end{enumerate}
\end{proposition}
\begin{proof} ~\\
    \begin{enumerate}[(i)]
        \item We have via \Cref{lem:subw-mgf}{(i)} that for every $i \in [n]$,
        \begin{equation*} 
            \E[\exp(\lambda X_i) \,|\, \F_{i-1}] \leq \exp\brb{4 e \wnormuni_{i}^2 \lambda^2} \qquad \forall \lambda \in \R \,.
        \end{equation*}
        Hence, the required result follows from \Cref{lem:martingale-conc-mgf-condition}{(i)} using that $r^2 \leq 4e \sum_{i=1}^n \wnormuni_i^2$.
        
        \item For $i \in [n]$, let $\ell_{i} = \wnormuni_{i} h_i$, where \[ h_i = \log^\theta \lrb{ \frac{e \sum_{j=1}^n \wnormuni_j^s}{\wnormuni_i^s \delta{'}} } \, \]
        for some $\delta{'} \in (0,1)$. Define $\tilde{X}_i = X_i \I\{X_i \leq\ell_{i}\}$ and $\tilde{S}_t = \sum_{i=1}^t \tilde{X}_i$.
        Note that for any $x>0$,
        \begin{align} \label{eq:subw-truncated-martingale}
            P\lrb{\bigcup_{t=1}^n \lcb{S_t \geq x} } \leq P\lrb{\bigcup_{t=1}^n \bcb{\tilde{S}_t \geq x} } + P\lrb{\bigcup_{i=1}^n \lcb{ X_i > \ell_{i}} } \,.
        \end{align}
        Starting with the second term, we perform a union bound and proceed in a similar manner to the proof of Proposition 11 in \citep{Madden2020}:
        \begin{align}
            P\lrb{\bigcup_{i=1}^n \lcb{ X_i > \ell_{i}} } &\leq \sum_{t=1}^n P\lrb{ X_i > \ell_{i} } \nonumber\\
            ~
            &= \sum_{i=1}^n P\lrb{ \exp\Brb{\lrb{X_i/\wnormuni_{i}}^{1/\theta}} > \exp\Brb{h_i^{1/\theta}} } \nonumber\\
            ~
            &\leq \sum_{i=1}^n \exp\Brb{-h_i^{1/\theta}} \E\Bsb{ \E\Bsb{ \exp\Brb{\lrb{X_i/\wnormuni_{i}}^{1/\theta}}\,\big|\, \F_{i-1} } } \nonumber\\
            ~
            &\leq 2 \sum_{i=1}^n \exp\Brb{-h_i^{1/\theta}} \nonumber\\
            ~ \label{eq:truncation-union-bound}
            &= \frac{2}{e} \sum_{i=1}^n  \frac{\wnormuni_i^s \delta{'}}{ \sum_{j=1}^n \wnormuni_j^s} = \frac{2}{e} \delta{'} \leq \delta{'} \,.
        \end{align}
        Returning to the first term in \eqref{eq:subw-truncated-martingale}, notice that for $i \in [n]$, \Cref{lem:subw-mgf}(iii) implies that
        \begin{equation*}
            \E\bsb{\exp\brb{\lambda \tilde{X}_i} \,|\, \F_{i-1} } \leq \exp\brb{a \wnormuni_{i}^2 \lambda^2} \qquad \forall \lambda \in \lsb{0, \frac{1}{2 h_i^{1-\frac{1}{\theta}} \wnormuni_{i}}} \,,
        \end{equation*}
        where\footnote{We have used the fact that $h_i \geq 1$ and $\theta \geq 1$ to bound the value of $a$ stated in the lemma.}
        $
            a = (2^{2 \theta} + 1) \Gamma(2\theta + 1) +  \frac{2^{3 \theta} \Gamma(3\theta + 1)}{6}
        $. In preparation for applying \Cref{lem:martingale-conc-mgf-condition}(ii), we study the term
        \begin{equation*}
            \max_{i \in [n]} \wnormuni_{i} h_i^{1-\frac{1}{\theta}}  
            =  \max_{i \in [n]} \wnormuni_{i}  \log^{\theta-1} \lrb{ \frac{e \sum_{j=1}^n \wnormuni_j^s}{\wnormuni_i^s \delta{'}} } \,.
        \end{equation*}
        Assuming that $s > 0$ and $\theta > 1$, let $w = \sum_{j=1}^n \wnormuni_j^s / \delta{'}$, and observe that $e^{1/s} w^{1/s} \geq w^{1/s} \geq \wnormuni_*$.
        Define $f(z) = z \log^{\theta-1}(e w / z^s)$, and let $\hat{z}_1 = \exp(1- \theta + 1/s) w^{1/s}$ and $\hat{z}_2 = e^{1/s} w^{1/s}$. By inspecting its first derivative,
        \[
            f'(z) = \log^{\theta-2}(e w / z^s) \brb{\log(e w / z^s) - s (\theta - 1) },    
        \]
        we observe that $f$ is increasing in $(0,\hat{z}_1)$ and decreasing in $(\hat{z}_1,\hat{z}_2)$. Hence, if $m_* \leq \hat{z}_1$; then, 
        \begin{equation*}
            \max_{i \in [n]} \wnormuni_{i}  \log^{\theta-1} \lrb{ \frac{e \sum_{j=1}^n \wnormuni_j^s}{\wnormuni_i^s \delta{'}} } =  \wnormuni_*  \log^{\theta-1} \lrb{ \frac{e \sum_{j=1}^n \wnormuni_j^s}{\wnormuni_*^s \delta{'}} }\,. 
        \end{equation*}
        Otherwise,
        \begin{equation*}
            \max_{i \in [n]} \wnormuni_{i}  \log^{\theta-1} \lrb{ \frac{e \sum_{j=1}^n \wnormuni_j^s}{\wnormuni_i^s \delta{'}} } 
            ~
            \leq \hat{z}_1  \log^{\theta-1} \lrb{ \frac{e \sum_{j=1}^n \wnormuni_j^s}{\hat{z}_1^s \delta{'}} }
            ~
            = \hat{z}_1  (s \theta - s)^{\theta-1}
            ~
            \leq \wnormuni_*  (s \theta - s)^{\theta-1} \,. 
        \end{equation*}
        Combing both cases yields that
        \begin{equation*}
            \max_{i \in [n]} \wnormuni_{i}  \log^{\theta-1} \lrb{ \frac{e \sum_{j=1}^n \wnormuni_j^s}{\wnormuni_i^s \delta{'}} } \leq \wnormuni_* \max \bbcb{ \log^{\theta-1} \lrb{ \frac{e \sum_{j=1}^n \wnormuni_j^s}{\wnormuni_*^s \delta{'}} } ,  (s \theta - s)^{\theta-1}}\,,
        \end{equation*}
        which, trivially, also holds when either $s=0$ or $\theta=1$. Subsequently, if we define $u_1 = a \sum_{i=1}^n \wnormuni_i^2$ and $u_2$ as twice the right-hand side of the above inequality, we can apply \Cref{lem:martingale-conc-mgf-condition}(ii) with $r^2 \leq u_1$ and $b \leq u_2$ to obtain that 
        \begin{equation*}
            P\lrb{\bigcup_{t=1}^n \lcb{\tilde{S}_t \geq x} } \leq \exp\bbrb{-\min\bbcb{\frac{x^2}{4 u_1}, \frac{x}{2 u_2}}} \,. 
        \end{equation*}
        Finally, by choosing $x = \sqrt{4 u_1 \log(1/\delta{'})} + 2 u_2 \log(1/\delta{'})$, we can upper bound the right-hand side of the above inequality with $\delta{'}$. Combining this with \eqref{eq:truncation-union-bound} and \eqref{eq:subw-truncated-martingale}, the required result follows after setting $\delta{'} = \delta/2$ and using that $a \leq 2^{3 \theta + 1} \Gamma(3\theta + 1)$.   
    \end{enumerate}    
\end{proof}

\subsection{Increments with a Bounded Moment Condition}
The following proposition, a weaker version of Corollary 3.2 in \citep{Rio2017}, is an analogue of \Cref{lem:subw-martingale-conc} when we only have the assumption that the increments of the martingale have a finite $p$-th absolute moment for some $p>2$.

\begin{proposition} \label{lem:fuk-nagaev-martingale-conc}
    Assume that $(X_i)_{i=1}^n$ is a martingale difference sequence adapted to filtration $\mathbb{F}=(\F_i)_{i=0}^n$, where $n$ is a positive integer, and let $S_t = \sum_{i=1}^t X_i$ for $t \in [n]$. Moreover, assume that there exists a constant $p>2$
    such that for each $i \in [n]$, $X_i$ satisfies 
    \begin{equation*}
        \E \Bsb{ \lrb{|X_i| / \pnormuni_{i}}^{p} \, \big| \, \F_{i-1}} \leq 1  
    \end{equation*}
    for some finite constant $\pnormuni_i >0$.
    Then, for any $\delta \in (0,1)$:
    \begin{equation*}
        P\lrb{\bigcup_{t=1}^n \bbcb{S_t >  \sqrt{2 \textstyle{\sum_{i=1}^n} \pnormuni^2_i \log(1/\delta)} + \brb{2+(p/3)} \brb{\textstyle{\sum_{i=1}^n} \pnormuni^p_i / \delta}^{1/p}  }} \leq \delta \,.
    \end{equation*} 
\end{proposition}
\begin{proof}
    Using that $\max\{0,X_i\} \leq |X_i|$ and $\E \lsb{ X_i^{2} \, | \, \F_{i-1}} \leq \brb{\E \lsb{ |X_i|^{p} \, | \, \F_{i-1}}}^{2/p} \leq \pnormuni^2_i$, the result follows from Corollary 3.2 and Remark 3.3 in \citep{Rio2017}.  
\end{proof}

\subsection{Auxiliary Lemmas}
For a random variable $X$, define $\|X\|_p \coloneqq (\E |X|^p)^{1/p}$ for $p > 0$. The following lemma relates $\|X-\E X\|_p$ to $\|X\|_p$ when $p\geq 1$.
\begin{lemma} \label{lem:centering-norms}
    Let $X$ be a random variable satisfying $\|X\|_p < \infty$ for some $p \geq 1$. Then, $\|X-\E X\|_p \leq 2 \|X\|_p$.
\end{lemma}
\begin{proof}
    We have that
    \begin{align*}
        \|X-\E X\|_p \leq \|X\|_p + \|\E X\|_p = \|X\|_p + |\E X|  \leq \|X\|_p + \E |X| \leq 2 \|X\|_p \,,
    \end{align*}
    where the first inequality is an application of the triangle's inequality as $\|\cdot\|_p$ is a norm for $p\geq1$, the second inequality is an application of Jensen's inequality, and the last inequality holds since $\|X\|_p$ is an increasing function in $p$. 
\end{proof}

The following lemma, similar in spirit to \citep[Lemma 26]{Madden2020}, allows us to reuse a standard argument when proving time-uniform concentration inequalities.
\begin{lemma} \label{lem:maximal}
    Fix a positive integer $n$ and assume that $(V_t)_{t=0}^n$ is a non-negative supermartingale adapted to filtration $(\F_{t})_{t=0}^n$ with $V_0 = 1$. Let $(A_t)_{t=1}^n$ be a sequence of events adapted to the same filtration, and assume that there exists a constant $\zeta > 0$ such that for any $t \in [n]$, it holds almost surely that $\I\{A_t\} \leq \zeta V_t$. Then,
    \begin{equation*}
        P\lrb{\bigcup_{t=1}^n A_t} \leq \zeta \,.
    \end{equation*}
\end{lemma}
\begin{proof}
    Define the stopping time $\tau = \min \bcb{ t \in [n] \:\colon\: \mathbb{I}\{A_t\} = 1}$, where $\min(\emptyset) = \infty$. Since $(V_t)_{t=0}^n$ is a supermartingale, the stopped process $(V_{t \land \tau})_{t=0}^n$ is also a supermartingale \cite[Theorem 10.9]{Williams1991}, where $t \land \tau = \min \{t, \tau\}$. This implies in particular that 
    \[ \E[V_{n \land \tau}] \leq \E[V_{0}] = 1 \,.\]
    Hence, 
    \begin{align*}
        P\lrb{\bigcup_{t=1}^n A_t} &= P\lrb{A_{n \land \tau}} = \E\bsb{\mathbb{I}\{A_{n \land \tau}\}} \leq  \zeta \E[V_{n \land \tau}] \leq \zeta \,,
    \end{align*}
    where the first inequality holds since
    \begin{align*}
        P\brb{\I\{A_{n \land \tau}\} > \zeta V_{n \land \tau}} \leq P\lrb{\bigcup_{t=1}^n \bcb{\I\{A_t\} > \zeta V_t}} = 0\,.
    \end{align*}
\end{proof}

The following lemma provides, via standard tools, concentration inequalities for sums of random variables enjoying sub-Gaussian or sub-exponential type bounds on their (conditional) moment generating functions.
\begin{lemma} \label{lem:martingale-conc-mgf-condition}
    Assume that $(X_i)_{i=1}^n$ is a sequence of random variables adapted to filtration $\mathbb{F}=(\F_i)_{i=0}^n$, where $n$ is a positive integer, and let $S_t = \sum_{i=1}^t X_i$ for $t \in [n]$.
    Moreover, let $(R_i)_{i=0}^n$ be a sequence of random variables adapted to the same filtration, and define $r^2 = \lno{\sum_{i=1}^n R_{i-1}^2}_\infty$.
    \begin{enumerate} [(i)]
        \item If for all $i \in [n]$, it holds that
        \begin{equation*} 
            \E[\exp(\lambda X_i) \,\mid\, \F_{i-1}] \leq \exp\brb{R_{i-1}^2 \lambda^2} \qquad \forall \lambda \in \R \,;
        \end{equation*}
        then, for all $x > 0$,
        \begin{equation*}
            P\lrb{\bigcup_{t=1}^n \lcb{S_t \geq x} } \leq \exp\bbrb{-\frac{x^2}{4 r^2}} \,. 
        \end{equation*}

        \item Let $(B_i)_{i=0}^n$ be an $\mathbb{F}$-adapted sequence of positive random variables,
        and define $b = \max_{i} \lno{B_{i}}_\infty$.
        If for all $i \in [n]$, it holds that
        \begin{equation*} 
            \E[\exp(\lambda X_i) \,\mid\, \F_{i-1}] \leq \exp\brb{R_{i-1}^2 \lambda^2} \qquad \forall \lambda \in \bbsb{0, \frac{1}{B_{i-1}}} \,;
        \end{equation*}
        then, for all $x > 0$,
        \begin{equation*}
            P\lrb{\bigcup_{t=1}^n \lcb{S_t \geq x} } \leq \exp\bbrb{-\min\bbcb{\frac{x^2}{4 r^2}, \frac{x}{2 b}}} \,. 
        \end{equation*}
    \end{enumerate}
\end{lemma}
\begin{proof}
    Define the set $\Lambda$ as $\R_{\geq0}$ in case \emph{(i)} and as $[0,1/b]$ in case \emph{(ii)}. Then, in either case, for any fixed $\lambda \in \Lambda$, the process $(V_t(\lambda))_{t=0}^n$, where 
    \[
        V_t(\lambda) = \prod_{i=1}^t \frac{\exp(\lambda X_i)}{\exp\brb{R_{i-1}^2 \lambda^2}}\,, \qquad \qquad V_0(\lambda) = 1
    \]
    is an $\mathbb{F}$-adapted non-negative supermartingale. 
    Moreover, notice that for any $t \in [n]$, it holds almost surely that
    \begin{align*}
        \I\{S_t \geq x\} &\leq \exp\lrb{\lambda S_t - \lambda x + \lambda^2 r^2 - \lambda^2 \sum_{i=1}^t R^2_{i-1}} \\
        &= \exp\lrb{ - \lambda x + \lambda^2 r^2} \exp\lrb{\lambda \sum_{i=1}^t X_i - \lambda^2 \sum_{i=1}^t R^2_{i-1}} \\
        &= \exp\lrb{ - \lambda x + \lambda^2 r^2} V_t(\lambda) \,.
    \end{align*}
    Consequently, \Cref{lem:maximal} implies that 
    \begin{equation*}
        P\lrb{\bigcup_{t=1}^n \lcb{S_t \geq x} } \leq \exp\lrb{ - \lambda x + \lambda^2 r^2}\,.
    \end{equation*}
    From this, the result in case \emph{(i)} follows by choosing $\lambda = \frac{x}{2 r^2}$, while the result in case \emph{(ii)} follows by choosing $\lambda = \min\bcb{\frac{x}{2 r^2},\frac{1}{b}}$ and using that $\frac{r^2}{b^2} \leq \frac{x}{2b}$ whenever $\frac{x}{2 r^2} \geq \frac{1}{b}$. 
\end{proof}
\end{document}